%% file: main.tex
\documentclass[11pt,letterpaper]{article}

\input{preamble}

\begin{document}

	\title{Neural Approximation by Function Composition:\\
	Rigidity and Doubly Exponential Convergence}
	\author{Wentao Huang\thanks{Department of Mathematics, The University of Hong Kong, Hong Kong, P.R. China. E-mail address:
\href{mailto:huangwt@hku.hk}{\textit{huangwt@hku.hk}}.}
\qquad
Haizhang Zhang\thanks{Corresponding author. School of Mathematics, Sun Yat-sen University, Guangzhou, P.R. China. E-mail address: \href{mailto:zhhaizh2@sysu.edu.cn}{\textit{zhhaizh2@sysu.edu.cn}}. Supported in part by the National Natural Science Foundation of China (Grant Nos. 12671642 and 12371103) and the Guangdong Basic and Applied Basic Research Foundation (Grant No. 2024A1515011194).}}
	\date{}
	\maketitle
\begin{abstract}

Deep neural networks approximate functions by composing affine maps with nonlinear activations, but how composition itself creates approximation power is not yet fully understood. We investigate a fundamental mechanism: geometrically weighted sums of iterates of a single scalar generator function. This mechanism underpins the classical tent-map construction of the function \(x - x^2\) and related recursive representations used by Yarotsky, W. E, et al., to analyze the approximation powers of deep neural networks.

First, we establish a rigidity theorem: for continuous piecewise linear generators with a finite number of segments, any \(C^3\) function that can be represented in this way is at most quadratic. For non-affine quadratic functions, the geometric factor is at least $1/4$. This result both reveals limitations of the tent-map approach and complements existing methods based on hierarchical bases and recursive polynomial constructions. Second, using an exact remainder identity as guidance, we construct a smooth generator whose iterates yield doubly exponential error decay in total depth for square approximation and, through multiplication modules, for each fixed polynomial. For power series with absolutely summable coefficients on \([-1,1]^d\), distributing depth according to monomial degree yields a uniform approximation error of order \(O(e^{-cL^{1/d}})\) on each interior cube. % Second, using an exact remainder identity as guidance, we construct a smooth generator such that for power series with absolutely summable coefficients, distributing depth according to the degree of each monomial results in a uniform approximation error of order \(O(e^{-cL^{1/d}})\). %These networks have depth at most \(L\), width at most four, size \(O(L)\), and incorporate skip connections and a single fixed non-polynomial activation function that is globally \(1\)-Lipschitz and \(C^1\) smooth, with weights and biases bounded independently of the desired accuracy. 
These findings demonstrate how generator dynamics and remainder estimates govern depth allocation and approximation rates of deep neural networks.
\end{abstract}



\input{sections/introduction}
\input{sections/rigidity}
\input{sections/generator}
\input{sections/approximation}
\input{sections/conclusion}

\section*{Declaration of AI use}

The original mathematical ideas and the initial manuscript were
developed by the authors without AI assistance. During subsequent
revisions, the authors used OpenAI Codex (GPT-5.6-Sol and GPT-6-Astra) to assist
with language editing, reference checking,
preparation of code for mathematical figures, and refinement of the
approximation error for the degree-dependent
allocation of network depth in Section~\ref{sec:Approx}.
All material retained from this assistance was independently checked
and, where necessary, revised by the authors. The authors assume
full responsibility for the accuracy, originality, and integrity
of the final manuscript.

%\section*{Conflict of interest}
% The authors declare that they have no conflict of interest.

{\small
\bibliographystyle{amsplain-full}
\bibliography{ref}
}
\end{document}

%% file: preamble.tex
\usepackage{verbatim}
\usepackage{float}
\usepackage{amsmath}
\usepackage{amsfonts}
\usepackage{amssymb}
\usepackage{amsthm}
\usepackage{graphicx}
\usepackage{tikz}
\usetikzlibrary{arrows.meta,positioning}
\usepackage{pgfplots}
\pgfplotsset{compat=1.17}
\usepackage{microtype}
\usepackage{xcolor}
\usepackage{comment}
\usepackage{hyperref}
\hypersetup{hidelinks}

\newtheoremstyle{templateplain}
  {6pt}{6pt}{\itshape}{}{\bfseries}{}{0.5em}
  {\thmname{#1}\thmnumber{ #2}\thmnote{ (#3)}}
\newtheoremstyle{templateupright}
  {6pt}{6pt}{\normalfont}{}{\bfseries}{}{0.5em}
  {\thmname{#1}\thmnumber{ #2}\thmnote{ (#3)}}

\theoremstyle{templateplain}
\newtheorem{theorem}{Theorem}[section]
\newtheorem{lemma}[theorem]{Lemma}
\newtheorem{proposition}[theorem]{Proposition}
\newtheorem{corollary}[theorem]{Corollary}
\theoremstyle{templateupright}

\newtheorem{definition}[theorem]{Definition}
\newtheorem{remark}[theorem]{Remark}

\numberwithin{equation}{section}
\numberwithin{figure}{section}
\numberwithin{table}{section}

\makeatletter
\renewenvironment{proof}[1][\proofname]{\par
  \pushQED{\qed}%
  \normalfont \topsep6\p@\@plus6\p@\relax
  \trivlist
  \item[\hskip\labelsep\itshape #1\@addpunct{:}]\ignorespaces
}{%
  \popQED\endtrivlist\@endpefalse
}
\makeatother

\newcommand{\R}{\mathbb{R}}

\newcommand{\PL}{\mathcal{PL}}
\newcommand{\G}{\mathcal{G}}
\newcommand{\GPL}{\mathcal{G}_{\mathrm{PL}}}
\newcommand{\Bdd}{\mathcal{B}}
\newcommand{\Self}{\mathfrak{S}}
\newcommand{\Id}{\operatorname{Id}}
\newcommand{\norm}[1]{\left\lVert #1\right\rVert}

\def \relu {{\mathrm{ReLU}}}

\makeatletter

\newcommand{\Rmnum}[1]{\expandafter\@slowromancap\romannumeral #1@}
\makeatother

%% file: sections/introduction.tex
\section{Introduction}
\label{sec:Intro}

Function representation and approximation play a central role in
mathematics and machine learning. Classical methods, including power
series, Fourier expansions, and wavelets, approximate a target by linear
combinations of prescribed functions. A feedforward neural network
constructs intermediate representations layer by layer by composing
affine maps with nonlinear activations. This raises a natural question:
how can composition improve approximation as depth increases?

In this paper, we study a construction that repeatedly applies a single
scalar function, called the generator, and forms geometrically weighted
sums of its iterates. We investigate how the choice of generator affects
the functions that can be represented and the convergence of their
approximations. Our starting point is the classical approximation of the
square function using the tent map
\begin{equation}
 s(x)=
 \begin{cases}
  2x,&0\leq x\leq1/2,\\
  2-2x,&1/2<x\leq1,
 \end{cases}
 \qquad
 s(x)=2x-4\relu(x-\tfrac12)\quad(x\in[0,1]),
 \label{eq:tent-map-definition}
\end{equation}
where \(\relu(t)=\max\{0,t\}\). Write
\(\varphi^{\circ j}\) for the \(j\)-fold self-composition of a map
\(\varphi\), and set \(\varphi^{\circ0}=\Id\). The identity used in
\cite{Liang,Yarotsky} is
\begin{equation}
 F(x):=x-x^2=\sum_{\ell=1}^{\infty}4^{-\ell}s^{\circ\ell}(x),
 \qquad x\in[0,1].
 \label{eq:tent-map-square-expansion}
\end{equation}
The partial sums interpolate \(F\) at the dyadic points \(k/2^L\);
see Figure~\ref{fig:tent-map-interpolation}. Reusing the same tent map
therefore gives a square approximation, from which the
difference-of-squares identity supplies a multiplication module.
E and Wang \cite{E} used this construction for fixed-width ReLU
approximation of low-dimensional analytic functions.

\begin{figure}[htbp]
 \centering
 \begin{tikzpicture}
  \begin{axis}[
   width=0.90\textwidth,
   height=0.36\textwidth,
   xmin=0,xmax=1,
   ymin=0,ymax=0.27,
   xlabel={\(x\)},
   ylabel={function value},
   xtick={0,0.1,...,1},
   ytick={0,0.05,0.10,0.15,0.20,0.25},
   yticklabels={\(0\),\(0.05\),\(0.10\),\(0.15\),\(0.20\),\(0.25\)},
   tick label style={font=\footnotesize},
   label style={font=\footnotesize},
   tick align=outside,
   tick style={black!70},
   legend style={font=\footnotesize,at={(0.5,1.04)},anchor=south,
                 legend columns=4,draw=none,fill=none,
                 inner xsep=0pt,inner ysep=1pt,
                 column sep=1.5em,row sep=2pt,
                 cells={anchor=west}},
   legend image post style={xscale=1.25},
   grid=major,
   grid style={gray!25,line width=0.35pt},
   axis line style={black!70},
   line width=1pt]
   \addlegendimage{black,line width=1.4pt}
   \addlegendentry{\(x-x^2\)}
   \addplot[red!75!black,densely dashed] coordinates
    {(0,0) (0.5,0.25) (1,0)};
   \addlegendentry{\(L=1\)}
   \addplot[blue!75!black,
            dash pattern=on 4.5pt off 1.8pt on 0.8pt off 1.8pt] coordinates
    {(0,0) (0.25,0.1875) (0.5,0.25) (0.75,0.1875) (1,0)};
   \addlegendentry{\(L=2\)}
   \addplot[teal!70!black,densely dotted,line width=1.15pt] coordinates
    {(0,0) (0.125,0.109375) (0.25,0.1875) (0.375,0.234375)
     (0.5,0.25) (0.625,0.234375) (0.75,0.1875)
     (0.875,0.109375) (1,0)};
   \addlegendentry{\(L=3\)}
   \addplot[black,line width=1.4pt,domain=0:1,samples=201,forget plot]
    {x-x^2};
  \end{axis}
 \end{tikzpicture}
 \caption{Successive piecewise linear interpolants
 \(\sum_{\ell=1}^{L}4^{-\ell}s^{\circ\ell}\) of \(x-x^2\).}
 \label{fig:tent-map-interpolation}
\end{figure}
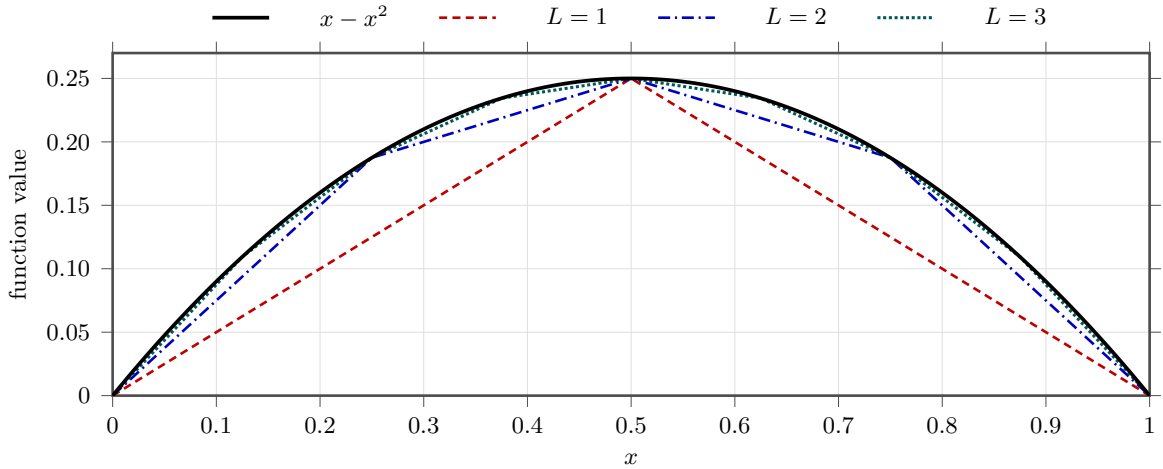

The expansion follows from the one-step relation
\(F=\tfrac14s+\tfrac14F\circ s\). More generally,
for a self-map \(\varphi:I\to I\) of a compact interval and \(|\beta|<1\),
\begin{equation}
 f=\alpha\sum_{\ell=1}^{\infty}\beta^{\ell-1}\varphi^{\circ\ell}
 \quad\Longleftrightarrow\quad
 f-\alpha\varphi=\beta f\circ\varphi
 \qquad(f\text{ bounded}).
 \label{eq:intro-functional-equation}
\end{equation}
Writing
\(S_L=\alpha\sum_{\ell=1}^L\beta^{\ell-1}\varphi^{\circ\ell}\),
iteration gives the exact remainder
\begin{equation}
 f(x)-S_L(x)=\beta^L f(\varphi^{\circ L}(x)).
 \label{eq:intro-exact-remainder}
\end{equation}
The series solution and this identity belong to the classical theory;
see \cite[Chapters~1--4]{KuczmaChoczewskiGer} and
\cite[Chapter~I]{SinghManhas} for the general framework, and
\cite[Chapter~3]{Despres} for the related neural-network viewpoint.
The remainder separates the geometric factor from the values of the
target along the iterates. It gives a common starting point for studying
representation limits, faster convergence, and depth allocation.

Our first results concern piecewise linear generators. By
Theorems~\ref{thm:notes-smooth-rigidity} and~\ref{thm:notes-beta-optimal},
a continuous piecewise linear generator with finitely many pieces can
produce a \(C^3\) output only if that output is a polynomial of degree at most two. A non-affine
quadratic requires \(\beta\geq1/4\), and its uniform truncation error is
\(\Theta(\beta^L)\) for each fixed representation. A fixed point where
the target is nonzero prevents the iterated target from contributing
additional uniform decay. The tent map attains both bounds.

He, Li, and Xu~\cite{He} give a hierarchical basis interpretation of
the tent-map construction and establish a related quadratic rigidity
result for the fixed tent map, without requiring geometric coefficients.
Our result allows
the piecewise linear generator to vary, while retaining geometric
weights. Despr\'es and Ancellin~\cite{DespresAncellin2020} construct
recursive representations of arbitrary univariate polynomials of the form
\(H=e_0+\sum_i\beta_iH\circ e_i\), using several piecewise linear
composition maps and a separate piecewise linear additive term;
see also \cite[Chapter~3]{Despres}.
These works also discuss the classical tent-map identity.
Our results concern the additional constraint that a single generator
\(\varphi\) supplies both the additive term \(\alpha\varphi\)
and the composition map in \(f\circ\varphi\).
Their polynomial existence results do not, by themselves, determine
the rigidity or sharp truncation rates under this constraint.
Our rigidity results do not limit polynomial approximation by general
ReLU networks.

We next use the remainder to design a smooth generator whose iterates
approach zero. For fixed admissible parameters \(\alpha>1\) and
\(0<\beta<1\), as specified in Section~\ref{sec:Overcome},
the generator satisfies
\(0\leq\varphi(x)\leq x^2/\alpha\) on \([-1,1]\).
The square module \(Q_n\), using \(n\) iterations and depth proportional
to \(n\), satisfies
\begin{equation}
 0\leq t^2-Q_n(t)\leq\varepsilon_n|t|^{2^{n+1}},
 \qquad
 \varepsilon_n=\beta^n\alpha^{-2(2^n-1)},\quad |t|\leq1.
 \label{eq:intro-local-remainder}
\end{equation}
This gives doubly exponential error decay in depth. The construction uses the fixed activation
\[
 \psi(t)=\sqrt{1+(t_+)^2}-1,
 \qquad t_+=\max\{t,0\}.
\]
This activation is nonpolynomial, belongs to \(C^1(\mathbb R)\),
and is globally \(1\)-Lipschitz. The weights and biases are bounded
independently of the approximation accuracy.

The local factor \(|t|^{2^{n+1}}\) connects this estimate to larger
networks. The associated multiplication modules satisfy
\(|\mathcal M_n(u,v)|\leq|uv|\), preserving the small magnitudes of
their factors. Composing them gives a doubly exponential error bound
in total depth for each fixed polynomial on \([-1,1]^d\)
(Theorem~\ref{thm:notes-multivariate-polynomial}). On an interior cube,
balanced monomial factorizations give small inputs to higher-degree
modules, allowing them to use less depth. Choosing module depths by
degree gives total depth and size \(O(p^d)\) for all monomials through
degree \(p\), for fixed \(d\).

Combining this allocation with power-series truncation gives uniform
error \(O(e^{-cL^{1/d}})\) on each interior cube for targets with
absolutely summable power-series coefficients on \([-1,1]^d\).
Theorem~\ref{thm:notes-multivariate-analytic} gives depth at most \(L\),
width at most four, and size \(O(L)\). The networks reuse earlier
outputs through skip connections, with all connection weights counted
in the size. Depth and width count both nonlinear and linear hidden
layers and neurons; width does not bound storage.

For fixed \(d\), our size exponent \(1/d\) exceeds the \(1/(2d)\)
exponent of E and Wang~\cite{E} in the same power-series and
interior-cube setting. Opschoor, Schwab, and
Zech~\cite[Theorems~3.6 and~3.10]{Opschoor2022} obtain size exponents
\(1/(d+1)\) for ReLU and \(1/d\) for RePU under holomorphic-extension
assumptions. Section~\ref{sec:Approx} compares these results, including
the differences in domains, norms, architectures, and depth bounds.

For suitable smooth activations, classical finite-difference
constructions give fixed-size multiplication approximations using
weights that grow as the error tends to zero;
see \cite[Remark~2.15]{Opschoor2022} and \cite{Pinkus}.
Shen, Yang, and Zhang~\cite{Shen3} obtain arbitrary accuracy with a
fixed number of neurons using a specially constructed activation.
Their result does not provide accuracy-independent bounds on weights
and biases.
Here we study error decay with depth under a fixed activation and
bounded coefficients. Our bounds concern this construction and do
not assert optimality over activations or network architectures.

Foundations of neural approximation include \cite{Barron,Cybenko};
broader accounts of
approximation and compositional structure appear in
\cite{Daubechies,Devore1,DoleanMontanelli2026,Elbrachter,
Goodfellow,LeCun,PetersenZech2024,Poggio}.
Related constructions use sparse grids \cite{Montanelli1}, the
Kolmogorov--Arnold representation \cite{Montanelli2}, and width--depth
tradeoffs \cite{Shen1,Shen2}.
Other work addresses sharp Sobolev and Besov approximation bounds
\cite{Siegel2023}, norm constraints \cite{Jiao},
convergence as depth grows \cite{XuZhang2022,XuZhang2024},
and convolutional universality \cite{Zhou1}.

The remainder of the paper is organized as follows. Section~\ref{sec:GWC} develops the functional equation and proves the
piecewise linear rigidity results. Section~\ref{sec:Overcome} constructs
the smooth square module and establishes its convergence rate with
bounded weights.
Section~\ref{sec:Approx} uses the local error estimates to build polynomial
networks and allocate depth for analytic approximation.

%% file: sections/rigidity.tex
\section{Geometrically weighted compositions and piecewise linear rigidity}
\label{sec:GWC}

The tent-map identity \eqref{eq:tent-map-square-expansion} is one instance of a general self-similar mechanism.
Let \(I\) be a compact interval and let \(\varphi\) be a function from $I$ to $I$, called a self-map of $I$.
Suppose that \(f\) admits the geometrically weighted expansion
\begin{equation}
 f(x)=\alpha\sum_{\ell=1}^{\infty}
      \beta^{\ell-1}\varphi^{\circ\ell}(x),
 \qquad \alpha\in\R,\quad |\beta|<1.
 \label{eq:notes-geometric-representation}
\end{equation}
All iterates remain in \(I\), so the series converges uniformly. For every
truncation level \(L\geq1\), 
\[
 \begin{aligned}
 f(x)-\alpha\sum_{\ell=1}^{L}
 \beta^{\ell-1}\varphi^{\circ\ell}(x)
 &=\alpha\sum_{j=1}^{\infty}
   \beta^{L+j-1}\varphi^{\circ(L+j)}(x)\\
 &=\beta^L f\bigl(\varphi^{\circ L}(x)\bigr).
 \end{aligned}
\]
The remainder is a scaled copy of \(f\), evaluated at the current
iterate \(\varphi^{\circ L}(x)\). Its decay depends on both the
geometric factor and the values of the target along these iterates.
Taking \(L=1\) recovers the one-step
equation~\eqref{eq:intro-functional-equation}.

We first formulate the representation for arbitrary self-maps of a
compact interval. We then show how requiring a continuous piecewise
linear generator restricts both the smooth outputs and their truncation
rates. For non-affine quadratic outputs, the iterates cannot provide
additional uniform decay in the remainder. This limitation motivates the design of a new generator in
Section~\ref{sec:Overcome}.

\subsection{The compositional class and its basic properties}

For a nondegenerate compact interval \(I=[a,b]\), write
\[
 \Self(I):=\{\varphi:I\to I\}
\]
for the set of all self-maps of \(I\), and let \(\Bdd(I)\) be the space of bounded real-valued functions on \(I\), equipped with the uniform norm.
Throughout this paper, \(\|\cdot\|_{L^\infty(I)}\) denotes the supremum norm, not the essential supremum. No continuity, piecewise linearity, or differentiability
is required in the following definition.

\begin{definition}[Geometrically weighted compositional class]
Define
\begin{equation}
 \G(I):=
 \left\{
  \alpha\sum_{\ell=1}^{\infty}
  \beta^{\ell-1}\varphi^{\circ\ell}:
  \alpha\in\R,\quad |\beta|<1,\quad \varphi\in\Self(I)
 \right\}.
 \label{eq:notes-geometric-class}
\end{equation}
The series is understood in \(\Bdd(I)\) with the uniform norm.
\end{definition}

The definition is meaningful because all iterates take values in the
compact interval \(I\). Each series in \eqref{eq:notes-geometric-class}
therefore converges absolutely and uniformly. If \(\varphi\) is continuous,
then every iterate is continuous and the uniform limit belongs to
\(C(I)\). Continuous and piecewise linear generators can thus be treated
as additional restrictions within the same framework.

The unrestricted class alone imposes little constraint on the output.
For example, \(\G([-1,1])=\Bdd([-1,1])\): for bounded \(f\), choose
\(\alpha\geq\max\{1,\|f\|_{L^{\infty}([-1,1])}\}\), set \(\varphi=f/\alpha\),
and take \(\beta=0\). The substantive questions concern restrictions
on the generator and on the representation parameters.

We first establish three elementary properties of functions in \(\mathcal{G}(I)\). Scaling changes only the leading coefficient. Affine normalization allows us to reduce the analysis to \([0,1]\), provided that the target function is shifted and rescaled accordingly. The
functional equation then replaces the infinite series by a one-step identity and gives its exact finite-depth remainder.

\begin{lemma}[Scaling]
\label{lem:notes-scaling}
If \(f\in\G(I)\) and \(\gamma\in\R\), then
\(\gamma f\in\G(I)\). More precisely, multiplying a representation of
\(f\) by \(\gamma\) replaces \(\alpha\) by \(\gamma\alpha\) and leaves
\(\beta\) and \(\varphi\) unchanged.
\end{lemma}

\begin{proof}
For any representation \eqref{eq:notes-geometric-representation},
\[
 \gamma f=(\gamma\alpha)
 \sum_{\ell=1}^{\infty}\beta^{\ell-1}\varphi^{\circ\ell},
\]
which has the required form.
\end{proof}

\begin{lemma}[Affine normalization of a representation]
\label{lem:notes-affine-normalization}
Suppose that \(f\in\mathcal{G}(I)\) and \(I=[a,b]\). Let
\[
 t(y)=\frac{y-a}{b-a},\qquad
 \chi=t\circ\varphi\circ t^{-1},
\]
and define
\begin{equation}
 \widetilde f(x)
 =\frac{f(t^{-1}(x))-\alpha a/(1-\beta)}{b-a},
 \qquad x\in[0,1].
 \label{eq:notes-normalized-target}
\end{equation}
Then \(\chi\in\Self([0,1])\) and \(\widetilde f\in\G([0,1])\), with
\begin{equation}
 \widetilde f(x)
 =\alpha\sum_{\ell=1}^{\infty}\beta^{\ell-1}\chi^{\circ\ell}(x).
 \label{eq:notes-normalized-representation}
\end{equation}
Continuity, finite piecewise linearity, and differentiability are
preserved by this affine normalization.
\end{lemma}

\begin{proof}
Since \(t\) maps \(I\) bijectively onto \([0,1]\), the conjugate
\(\chi\) is a self-map of \([0,1]\). For every \(\ell\geq1\),
\[
 \chi^{\circ\ell}(x)
 =\left(t\circ\varphi^{\circ\ell}\circ t^{-1}\right)(x)
 =\frac{\varphi^{\circ\ell}(t^{-1}(x))-a}{b-a}.
\]
Consequently,
\[
 \begin{aligned}
 \alpha\sum_{\ell=1}^{\infty}\beta^{\ell-1}\chi^{\circ\ell}(x)
 &=\frac{1}{b-a}
   \left(f(t^{-1}(x))
    -\alpha a\sum_{\ell=1}^{\infty}\beta^{\ell-1}\right)\\
 &=\frac{f(t^{-1}(x))-\alpha a/(1-\beta)}{b-a}
 =\widetilde f(x).
 \end{aligned}
\]
This proves \eqref{eq:notes-normalized-representation}. The regularity
assertions follow from the affine changes of variables and target values.
\end{proof}

\begin{lemma}[Functional-equation characterization]
\label{lem:notes-functional-equation}
For \(f\in\Bdd(I)\), \(\alpha\in\R\), \(|\beta|<1\), and
\(\varphi\in\Self(I)\), representation
\eqref{eq:notes-geometric-representation} holds if and only if
\begin{equation}
 f(x)-\alpha\varphi(x)=\beta f(\varphi(x)),
 \qquad x\in I.
 \label{eq:notes-functional-equation}
\end{equation}
For every integer \(L\geq1\), either condition gives
\begin{equation}
 f(x)-\alpha\sum_{\ell=1}^L
       \beta^{\ell-1}\varphi^{\circ\ell}(x)
 =\beta^L f(\varphi^{\circ L}(x)),
 \label{eq:notes-exact-remainder}
\end{equation}
and hence
\begin{equation}
 \left\|f-\alpha\sum_{\ell=1}^L
       \beta^{\ell-1}\varphi^{\circ\ell}\right\|_{L^\infty(I)}
 \leq|\beta|^L\|f\circ \varphi^{\circ L}\|_{L^\infty(I)}
 \le |\beta|^L \|f\|_{L^\infty(I)}.
 \label{eq:notes-remainder-bound}
\end{equation}
\end{lemma}

\begin{proof}
Suppose first that \eqref{eq:notes-geometric-representation} holds.
Separating its first term and using uniform convergence, we obtain
\[
 \begin{aligned}
 f(x)-\alpha\varphi(x)
 &=\alpha\sum_{\ell=2}^{\infty}
   \beta^{\ell-1}\varphi^{\circ\ell}(x)\\
 &=\beta f(\varphi(x)),
 \end{aligned}
\]
which is \eqref{eq:notes-functional-equation} with the same parameters.

Conversely, assume that \eqref{eq:notes-functional-equation} holds.
Substituting \(\varphi^{\circ(\ell-1)}(x)\) for \(x\) and multiplying
by \(\beta^{\ell-1}\) gives
\[
 \beta^{\ell-1}f(\varphi^{\circ(\ell-1)}(x))
 -\alpha\beta^{\ell-1}\varphi^{\circ\ell}(x)
 =\beta^\ell f(\varphi^{\circ\ell}(x)).
\]
Summing over \(\ell=1,\ldots,L\) cancels all intermediate terms and
yields \eqref{eq:notes-exact-remainder}. Since
\(\varphi^{\circ L}(I)\subseteq I\),
\[
 \|\beta^L f\circ\varphi^{\circ L}\|_{L^\infty(I)}
 \leq |\beta|^L\|f\|_{L^\infty(I)}\longrightarrow0.
\]
The partial sums therefore converge uniformly to \(f\), proving
\eqref{eq:notes-geometric-representation} and the stated bound.
\end{proof}

The characterization is for a fixed choice of \(\alpha\), \(\beta\), and
\(\varphi\), not merely an assertion that some representation exists.
This distinction matters below, where we derive restrictions on the
generator and on the geometric factor from the functional equation.
The exact remainder also separates two possible sources of decay:
the factor \(\beta^L\) and the value of \(f\) along the iterated map.

\subsection{The piecewise linear generator bottleneck}

We now impose piecewise linearity on the generator.  Let
\(\PL(I)\subset\Self(I)\) be the set of continuous self-maps
\(\varphi:I\to I\) for which there exists a finite partition
\(a=x_0<x_1<\cdots<x_m=b\) such that \(\varphi\) is affine on every
\([x_{i-1},x_i]\).  Define the subclass
\begin{equation}
 \GPL(I)
 :=
 \left\{
  \alpha\sum_{\ell=1}^{\infty}
  \beta^{\ell-1}\varphi^{\circ\ell}:
  \alpha\in\R,\quad |\beta|<1,\quad
  \varphi\in\PL(I)
 \right\}
 \subseteq\G(I).
 \label{eq:notes-pl-subclass}
\end{equation}
In the next theorem, the regularity assumption is imposed on
the output \(f\). Although iterating a piecewise linear map can create many new knots, the functional equation strongly constrains which \(C^{3}\) functions the resulting series can represent.

\begin{theorem}[Smooth rigidity]
\label{thm:notes-smooth-rigidity}
If
\[
 f\in\GPL([0,1])\cap C^3([0,1]),
\]
then \(f\) is a polynomial of degree at most two, including the affine and
constant cases.
\end{theorem}

\begin{proof}
By the definition of \(\GPL([0,1])\) and 
Lemma~\ref{lem:notes-functional-equation}, there exist
\(\alpha\in\R\), \(|\beta|<1\), and
\(\varphi\in\PL([0,1])\) such that
\begin{equation}
 f-\alpha\varphi=\beta f\circ\varphi.
 \label{eq:notes-rigidity-fe}
\end{equation}
If \(\beta=0\), then \(f=\alpha\varphi\).  Since a piecewise linear
\(C^3\) function is affine, the conclusion follows.  If \(\varphi\) is
affine on \([0,1]\), each iterate of \(\varphi\) is affine; the uniformly
convergent series defining \(f\) is therefore affine as well.  We may thus
assume that \(\beta\neq0\) and that \(\varphi\) has at least one interior
knot.

We show that the third derivative of \(f\) vanishes, first at the knots and their
preimages, and then on the intervals that avoid them. Let \(m\ge 2\) and
\[
 0=x_0<x_1<\cdots<x_m=1
\]
be a minimal partition on which
\(\varphi(x)=k_i x+c_i\) for
\(x\in[x_{i-1},x_i]\), and let
\(\mathcal K=\{x_1,\ldots,x_{m-1}\}\) be the set of interior knots.  Put \(u=f'''\).  On every open
linearity interval, differentiating \eqref{eq:notes-rigidity-fe} three
times gives
\begin{equation}
 u(x)=\beta k_i^3u(\varphi(x)).
 \label{eq:notes-u-equation}
\end{equation}
At a knot \(\xi=x_i\), continuity of \(u\) and the left and right limits
in \eqref{eq:notes-u-equation} imply
\[
 u(\xi)=\beta k_i^3u(\varphi(\xi))
       =\beta k_{i+1}^3u(\varphi(\xi)).
\]
The partition is minimal, so \(k_i\neq k_{i+1}\).  It follows that
\(u(\varphi(\xi))=u(\xi)=0\).  Hence \(u=0\) on \(\mathcal K\).

Consider the set of points in $(0,1)$ whose forward orbits meet
the knot set $\mathcal K$,
\[
E=\bigcup_{\ell=0}^{\infty}
\bigl\{x\in(0,1):\varphi^{\circ\ell}(x)\in\mathcal K\bigr\}.
\]
Iterating \eqref{eq:notes-u-equation} along an orbit up to its
first visit to $\mathcal K$ shows that \(u=0\) on \(E\).
By continuity, $u=0$ on $\overline E$ as well.

The set $E$ need not be dense in \([0,1]\). To handle the remaining
points, let $J$ be any connected component of
$(0,1)\setminus\overline E$. Since this complement is open,
$J$ is a nonempty open interval, so $|J|>0$.
For every $l\geq0$, the image $\varphi^{\circ l}(J)$ is an
interval disjoint from $\mathcal K$; otherwise, some point
of $J$ would belong to $E$. Thus $\varphi$ is affine on each
$\varphi^{\circ l}(J)$, and induction shows that every iterate
of $\varphi$ is affine on $J$. In particular, for each
$\ell\geq1$, there are constants $K_\ell,C_\ell$ such that
\[
\varphi^{\circ\ell}(x)=K_\ell x+C_\ell,
\qquad x\in J.
\]
Since $\varphi^{\circ\ell}(J)\subseteq[0,1]$, taking diameters
gives $|K_\ell|\,|J|\leq1$.
Iterating \eqref{eq:notes-rigidity-fe} $\ell$ times and
differentiating the resulting identity three times on $J$,
we obtain
\begin{align*}
|u(x)|
&=|\beta|^\ell |K_\ell|^3
  |u(K_\ell x+C_\ell)|\\
&\leq |\beta|^\ell |K_\ell|^3
  \norm{u}_{L^\infty([0,1])}\\
&\leq |\beta|^\ell |J|^{-3}
  \norm{u}_{L^\infty([0,1])},
\qquad x\in J.
\end{align*}
Since $|\beta|<1$, letting $\ell\to\infty$ yields $u=0$ on $J$.
As $J$ is arbitrary and $u=0$ on $\overline E$, we conclude
that $u=0$ on $(0,1)$, and hence on $[0,1]$ by continuity.
Thus $f'''=0$ on $[0,1]$, so $f$ is a polynomial of degree
at most two.
\end{proof}

\begin{remark}
In \cite{He}, the generator \(s\) is fixed while the expansion coefficients
are arbitrary.  Here the coefficients are geometrically constrained,
whereas the continuous piecewise linear self-map \(\varphi\) may vary.  This
distinction separates the result in \cite{He} from the rigidity theorem
above.
\end{remark}

\begin{corollary}
    \label{cor:smooth-rigidity}
Let \(I=[a,b]\), with \(a<b\).  If
\[
 f\in\GPL(I)\cap C^3(I),
\]
then \(f\) is a polynomial of degree at most two.
\end{corollary}

\begin{proof}
Choose a representation
\[
 f=\alpha\sum_{\ell=1}^{\infty}
 \beta^{\ell-1}\varphi^{\circ\ell},
 \qquad \varphi\in\PL(I).
\]
With \(t\), \(\chi\), and \(\widetilde f\) as in
Lemma~\ref{lem:notes-affine-normalization}, affine conjugation gives
\(\chi\in\PL([0,1])\), while
\(\widetilde f\in C^3([0,1])\).  Thus
\(\widetilde f\in\GPL([0,1])\cap C^3([0,1])\), and
Theorem~\ref{thm:notes-smooth-rigidity} shows that \(\widetilde f\) is a
polynomial of degree at most two.  Finally,
\[
 f(y)=(b-a)\widetilde f(t(y))+\frac{\alpha a}{1-\beta},
 \qquad y\in I,
\]
so \(f\) also has degree at most two.
\end{proof}

We next quantify the truncation rate for the non-affine quadratic
outputs allowed by Theorem~\ref{thm:notes-smooth-rigidity}.
A lower bound on \(\beta\) alone would not exclude faster convergence,
since \(f\circ\varphi^{\circ L}\) could still decay. The next result
shows both that \(\beta\geq1/4\) and that the uniform norm of this
iterated target stays bounded below by a positive constant.
Together, these bounds determine the actual uniform truncation rate
for each fixed representation.

\begin{samepage}
\begin{theorem}[Optimal geometric factor and sharp truncation rate]
\label{thm:notes-beta-optimal}
Let
\[
 f(x)=c_2x^2+c_1x+c_0,
 \qquad c_2\neq0.
\]
Suppose that, for some \(\alpha\in\R\), \(|\beta|<1\), and
\(\varphi\in\PL([0,1])\),
\[
 f(x)=\alpha\sum_{\ell=1}^{\infty}
 \beta^{\ell-1}\varphi^{\circ\ell}(x).
\]
Then
\begin{equation}
 \beta\geq\frac14.
 \label{eq:notes-beta-lower-bound}
\end{equation}
Moreover, \(\varphi\) has a fixed point \(p\in(0,1]\) with
\(f(p)\neq0\). Hence, with \(c_f:=|f(p)|>0\),
\[
 \left\|f\circ\varphi^{\circ L}\right\|_{L^\infty([0,1])}
 \geq c_f,\qquad L\geq1.
\]
Consequently, the uniform truncation error satisfies
\begin{equation}
 c_f\beta^L
 \leq
 \left\|f-\alpha\sum_{\ell=1}^L
       \beta^{\ell-1}\varphi^{\circ\ell}\right\|_{L^\infty([0,1])}
 \leq\|f\|_{L^\infty([0,1])}\beta^L,
 \qquad L\geq1.
 \label{eq:notes-pl-actual-rate}
\end{equation}
Thus, for each fixed representation, the uniform truncation error
is of order \(\beta^L\). The iterated target
\(f\circ\varphi^{\circ L}\) provides no additional uniform decay.
Since \(\beta\geq1/4\), the truncation error cannot decay faster
than a positive multiple of \(4^{-L}\).
The constant \(c_f\) may depend on the fixed representation.
\end{theorem}
\end{samepage}

\begin{proof}
By Lemma~\ref{lem:notes-scaling}, we may divide by \(c_2\) and assume
\(f(x)=x^2+c_1x+c_0\).  The case \(\beta=0\) would imply
\(f=\alpha\varphi\), which is impossible for a non-affine quadratic.
The functional equation becomes
\[
 \beta\varphi(x)^2+(\alpha+\beta c_1)\varphi(x)
 +\beta c_0-f(x)=0.
\]
Equivalently,
\begin{equation}
 \bigl(2\beta\varphi(x)+\alpha+\beta c_1\bigr)^2=D(x),
 \label{eq:notes-discriminant-identity}
\end{equation}
where
\[
 D(x)=(\alpha+\beta c_1)^2
      -4\beta\bigl(\beta c_0-f(x)\bigr).
\]
On any nondegenerate linearity interval of \(\varphi\), the left-hand side
of \eqref{eq:notes-discriminant-identity} is the square of an affine
function.  Since the equality holds on a nondegenerate interval, the
polynomial identity theorem implies that the global quadratic polynomial
\(D\) is a perfect square.  Its leading
coefficient is \(4\beta\), so necessarily \(\beta>0\), and
\[
 D(x)=q(x)^2
\]
for an affine function \(q\) whose slope has absolute value
\(2\sqrt\beta\).

Equation \eqref{eq:notes-discriminant-identity} gives
\(2\beta\varphi(x)+\alpha+\beta c_1=\pm q(x)\).  Continuity permits a
change of sign only at the at most one zero of the affine function \(q\).
Hence \(\varphi\) has at most two affine pieces, and the absolute value of
its slope on each piece is \(\beta^{-1/2}\).  At least one of the
pieces has length at least \(1/2\).  Since \(\varphi([0,1])\subseteq[0,1]\),
the variation of \(\varphi\) over that piece cannot exceed \(1\).  Therefore
\[
 \frac12\beta^{-1/2}\leq1,
\]
which is equivalent to \eqref{eq:notes-beta-lower-bound}.

For the remainder estimate, return to the original \(f\) and \(\alpha\);
the conclusions about \(\varphi\) and \(\beta\) are unchanged. Write
\(s=\beta^{-1/2}>1\). The self-map \(\varphi\) has a fixed point
\(p\in(0,1]\). Indeed, if \(\varphi(0)>0\), this follows by applying
the intermediate value theorem to \(\varphi(x)-x\). If
\(\varphi(0)=0\), the first affine piece must have slope \(s\), since
\(\varphi\) is nonnegative. Thus \(\varphi(x)>x\) for small positive
\(x\), while \(\varphi(1)\leq1\), giving the same conclusion.
Also \(\alpha\ne0\); otherwise the functional equation and
\(|\beta|<1\) would imply \(f=0\). Evaluating that equation at \(p\)
gives
\[
 f(p)=\frac{\alpha p}{1-\beta}\ne0.
\]
Set \(c_f=|f(p)|=|\alpha|p/(1-\beta)>0\).
Since \(\varphi^{\circ L}(p)=p\), we have
\(\|f\circ\varphi^{\circ L}\|_{L^\infty([0,1])}\geq c_f\)
for every \(L\geq1\). The self-map property also gives the upper
bound \(\|f\circ\varphi^{\circ L}\|_{L^\infty([0,1])}
\leq\|f\|_{L^\infty([0,1])}\).
Multiplying these bounds by \(\beta^L\) and using the exact remainder
\eqref{eq:notes-exact-remainder} proves \eqref{eq:notes-pl-actual-rate}.
\end{proof}

\begin{remark}
Theorems~\ref{thm:notes-smooth-rigidity} and~\ref{thm:notes-beta-optimal}
concern the piecewise linear subclass \(\GPL([0,1])\) in
\eqref{eq:notes-pl-subclass}. They impose no such restrictions on the
full class \(\G([0,1])\) or on polynomial approximation by general
ReLU networks. Within this subclass, the classical tent map shows that
both the geometric-factor bound and the truncation-rate bound are sharp:
for \(F(x)=x-x^2\), surjectivity of every iterate gives the exact norm
error \(4^{-L}\|F\|_{L^\infty([0,1])}=4^{-L-1}\).
\end{remark}

Theorem~\ref{thm:notes-beta-optimal} rules out additional uniform decay
from \(f\circ\varphi^{\circ L}\) for these piecewise linear
representations. The next section constructs a generator whose iterates
approach a zero of the target, so this composed target tends uniformly
to zero and accelerates the truncation beyond the geometric weights.

% ============================================ %

%% file: sections/generator.tex
\section{A differentiable generator with quadratically convergent iterates}
\label{sec:Overcome}

We now seek additional uniform decay in the iterated target
\(f\circ\varphi^{\circ L}\) within the exact remainder
\eqref{eq:notes-exact-remainder}. For \(f(x)=x^2\) on \([-1,1]\), we construct
a smooth self-map satisfying \(0\leq\varphi(x)\leq x^2/\alpha\),
with \(\alpha>1\). Repeated composition then drives the target values
in the remainder to zero at a doubly exponential rate, while the
coefficients remain geometrically weighted. The generator is realized
using one fixed activation, with the two parameters entering only
through affine coefficients.

\subsection{Existence of a generator}

For a prescribed target \(f\), we first seek a continuous self-map
\(\varphi\) satisfying
\begin{equation}
 f(x)-\alpha\varphi(x)=\beta f(\varphi(x)).
 \label{eq:notes-decoupled-fe}
\end{equation}
We work on \(I=[-1,1]\) to accommodate targets of either sign.
The following proposition shows that, for every \(0<\beta<1\),
such a generator exists whenever \(f\) is Lipschitz and
\(\alpha\) is sufficiently large.

\begin{proposition}[Existence of a self-map solution]
\label{thm:notes-fixed-point}
Let \(I=[-1,1]\), and let \(f:I\to\mathbb R\) be Lipschitz with
constant \(L_f\).  If \(0<\beta<1\) and
\begin{equation}
 \alpha>\max\{\beta L_f,(1+\beta)\|f\|_{L^\infty(I)}\},
 \label{eq:notes-alpha-condition}
\end{equation}
then \eqref{eq:notes-decoupled-fe} has a unique continuous self-map
solution.  This solution satisfies
\[
 f=\alpha\sum_{\ell=1}^\infty\beta^{\ell-1}\varphi^{\circ \ell},
 \qquad
 \left\|f-\alpha\sum_{\ell=1}^L\beta^{\ell-1}\varphi^{\circ \ell}\right\|_{L^\infty(I)}
 \leq\beta^L\|f\circ \varphi^{\circ L}\|_{L^\infty(I)}.
\]
\end{proposition}

\begin{proof}
Let \(C(I,I):=\{u\in C(I,\mathbb R):u(I)\subseteq I\}\), equipped
with the uniform metric
\(d_\infty(u,v)=\|u-v\|_{L^\infty(I)}\).
Since \(I=[-1,1]\), this set is the closed unit ball of the Banach space
\(C(I,\mathbb R)\) under the uniform norm, and hence is a complete metric
space. Define \(Tu=\alpha^{-1}(f-\beta f\circ u)\). Condition \eqref{eq:notes-alpha-condition} gives \(\|Tu\|_{L^\infty(I)}<1\) and \(\|Tu-Tv\|_{L^\infty(I)}\leq(\beta L_f/\alpha)\|u-v\|_{L^\infty(I)}\).
Banach's fixed-point theorem gives the unique solution \(\varphi\). The representation
and error estimate follow from Lemma~\ref{lem:notes-functional-equation}
and \eqref{eq:notes-exact-remainder}.
\end{proof}

The generator in Proposition~\ref{thm:notes-fixed-point} depends on the
target \(f\), so this existence result alone does not provide a network
with a prescribed activation. For the square
function, however, the generator is explicit and can be realized using
a single fixed activation for every admissible parameter pair
\((\alpha,\beta)\). The resulting square approximation will be used to
approximate multiplication and, in turn, to construct polynomial
approximants in Section~\ref{sec:Approx}.

\subsection{Square approximation with a fixed activation}

Throughout the remaining sections, fix
\begin{equation}
 0<\beta<1,\qquad \alpha>\max\{2\beta,1+\beta\}.
 \label{eq:notes-square-alpha}
\end{equation}
The nonnegative solution of \eqref{eq:notes-decoupled-fe} for \(f(x)=x^2\)
is
\begin{equation}
 \varphi(x)=\frac{\sqrt{\alpha^2+4\beta x^2}-\alpha}{2\beta}
           =\frac{2x^2}{\sqrt{\alpha^2+4\beta x^2}+\alpha}.
 \label{eq:notes-explicit-varphi}
\end{equation}
It is smooth, even, nonnegative, and satisfies
\begin{equation}
 x^2=\alpha\varphi(x)+\beta\varphi(x)^2,
 \qquad 0\leq\varphi(x)\leq\frac{x^2}{\alpha}.
 \label{eq:notes-square-fe}
\end{equation}
In particular, \(\varphi([-1,1])\subseteq[0,1/\alpha]\subset[0,1]\),
so \(\varphi\) also preserves \([0,1]\). And we will use the quadratic bound in \eqref{eq:notes-square-fe}
to estimate the iterates of \(\varphi\) and the resulting
truncation error.

Define
\begin{equation}
 \psi(t)=
 \begin{cases}
 \sqrt{1+t^2}-1,&t\geq0,\\
 0,&t<0,
 \end{cases}
 \qquad
 a=\frac{\alpha}{2\beta},\qquad b=\frac{2\sqrt\beta}{\alpha}.
 \label{eq:notes-smooth-activation}
\end{equation}
The activation \(\psi\) is nonpolynomial, belongs to \(C^1(\mathbb R)\),
and is globally \(1\)-Lipschitz.  Indeed, its derivative is zero on the
negative half-line and equals \(t/\sqrt{1+t^2}\) for \(t>0\), with common
limit zero at the origin.  The generator has the realization
\begin{equation}
 \varphi(x)=a\bigl[\psi(bx)+\psi(-bx)\bigr].
 \label{eq:notes-varphi-from-sigma}
\end{equation}
This realization uses a single fixed activation \(\psi\), with
\(\alpha\) and \(\beta\) entering only through the affine coefficients.
For every integer \(\ell\geq1\), its iterates satisfy
\[
 \varphi^{\circ\ell}(x)
 =\bigl(a\psi(b\,\cdot)\bigr)^{\circ\ell}(x)
  +\bigl(a\psi(b\,\cdot)\bigr)^{\circ\ell}(-x),
 \qquad x\in[-1,1].
\]

For \(L\geq1\), consider the partial sums
\begin{equation}
 Q_L(x):=\alpha\sum_{\ell=1}^L\beta^{\ell-1}\varphi^{\circ\ell}(x).
 \label{eq:notes-square-network}
\end{equation}
Figure~\ref{fig:square-approximation} illustrates the approximation for
\((\alpha,\beta)=(2,1/2)\). The error decreases with \(L\) and
vanishes more rapidly near the origin. We next describe the networks
that realize these partial sums and specify their resource counts.

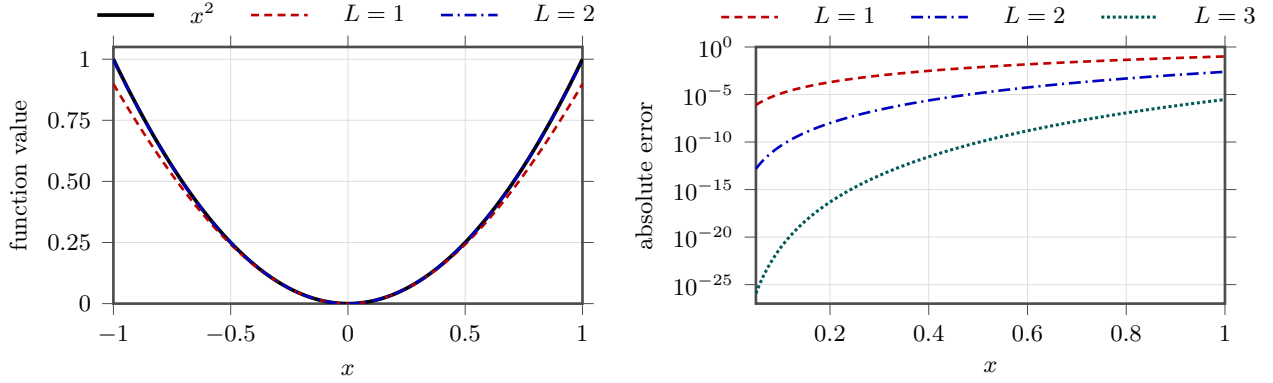
\begin{figure}[H]
 \centering
 \input{figures/square-approximation-plot}
 \caption{Square approximation with \((\alpha,\beta)=(2,1/2)\).
 Left: \(x^2\) and the partial sums \(Q_1,Q_2\) on \([-1,1]\).
 Right: the errors \(x^2-Q_L(x)\) for \(L=1,2,3\) on
 \(0.05\leq x\leq1\), with a logarithmic vertical axis.
 The errors are even and vanish at \(x=0\).}
 \label{fig:square-approximation}
\end{figure}

\begin{definition}[Compositional Neural Networks]
\label{def:notes-network-model}
A \(\psi\)-compositional network is a feedforward network with
alternating nonlinear and linear hidden layers, assembled from smaller
networks called \emph{modules}. Skip connections allow inputs and earlier
neuron outputs to be reused.

For an input \(\boldsymbol x\in\mathbb R^d\), set
\(\boldsymbol h_0=\boldsymbol x\). With \(H\) hidden layers, write
\[
 \begin{aligned}
 \boldsymbol h_\ell
 &=\rho_\ell(A_\ell\boldsymbol u_\ell+\boldsymbol b_\ell),
       \qquad 1\leq\ell\leq H,\\
 \mathcal N(\boldsymbol x)
 &=\boldsymbol a^{\mathsf T}\boldsymbol u_{\mathrm{out}}+c,
 \qquad
 \rho_\ell=
 \begin{cases}
  \psi,&\ell\text{ odd},\\
  \operatorname{Id},&\ell\text{ even}.
 \end{cases}
 \end{aligned}
\]
The activation acts coordinatewise. The vector \(\boldsymbol u_\ell\)
collects selected input coordinates and earlier neuron outputs, each
listed once; these selections are fixed by the network architecture.
The vector \(\boldsymbol u_{\mathrm{out}}\) is chosen in the same way.
For example, in a network consisting of a single module, the connections
after the first nonlinear--linear pair can take the form
\[
 \boldsymbol u_{2r-1}=\boldsymbol h_{2r-3},
 \qquad
 \boldsymbol u_{2r}=
 \begin{bmatrix}\boldsymbol h_{2r-1}\\\boldsymbol h_{2r-2}\end{bmatrix},
 \qquad r\geq2,
\]
whenever the indicated layers are present. The nonlinear layer reuses
the preceding nonlinear output, and the linear layer combines the new
nonlinear output with the preceding linear output.

The \emph{depth} is the number \(H\) of hidden layers. The \emph{width}
\(W\) counts all neurons in a hidden layer, including linear neurons;
the input and final output layers are excluded. The \emph{size} \(S\)
counts all nonzero weights and biases:
\[
 \begin{gathered}
 W=\max_{1\leq\ell\leq H}\dim\boldsymbol h_\ell,\\
 S=\sum_{\ell=1}^{H}
       \bigl(\operatorname{nnz}(A_\ell)
             +\operatorname{nnz}(\boldsymbol b_\ell)\bigr)
       +\operatorname{nnz}(\boldsymbol a)+\operatorname{nnz}(c).
 \end{gathered}
\]
Here \(\operatorname{nnz}\) counts nonzero entries, also for scalars;
set \(W=0\) when \(H=0\).
We say that the weights and biases are bounded by \(B>0\) if every
entry of \(A_\ell\), \(\boldsymbol b_\ell\), and \(\boldsymbol a\),
and the output bias \(c\), has absolute value at most \(B\).
\end{definition}

Each nonzero weight or bias is counted separately in \(S\), including
weights on skip connections. Reusing an existing output adds only the
weights on its new connections. When a square or multiplication module
is used inside a larger network, its output becomes a linear hidden neuron and is counted
in both depth and width. Stored outputs reused through skip connections
do not add neurons to later layers.

Connections between multiplication modules are prescribed by the
monomial factorization rule in Section~\ref{sec:Approx}; each module
receives the outputs of its two specified factors.
Figure~\ref{fig:network-model} shows the pattern within a square module
for \(Q_2\).

\begingroup
\setlength{\intextsep}{6pt}
\begin{figure}[H]
 \centering
 \includegraphics[width=0.90\linewidth,trim=10bp 0bp 10bp 0bp,clip]{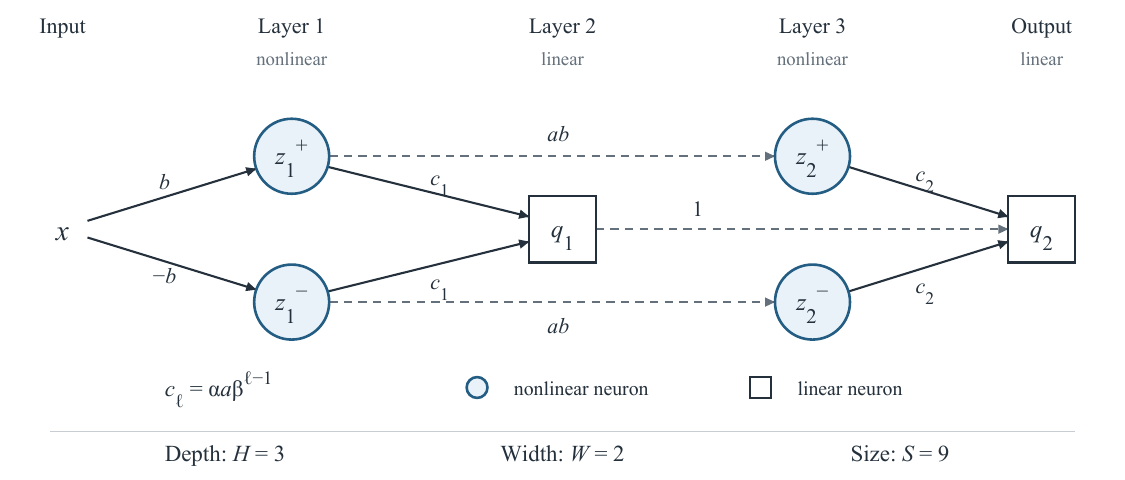}
 \caption{An alternating realization of \(Q_2\), with
 \(c_\ell=\alpha a\beta^{\ell-1}\).
 Circles denote nonlinear neurons and squares denote linear neurons.
 Dashed skip connections carry the branch values and the running sum
 across one intervening layer. All three hidden layers contribute to depth.}
 \label{fig:network-model}
\end{figure}
\endgroup

For general \(L\), the partial sum \(Q_L\) in
\eqref{eq:notes-square-network} has an alternating realization: a nonlinear layer updates the two generator
branches, and a linear layer adds their contribution to the running sum.
Here \(L\) counts generator iterations. Set
\[
 \varepsilon_L:=\beta^L\alpha^{-2(2^L-1)}.
\]
Starting from \(q_0=0\), define
\begin{equation}
 \begin{aligned}
 z_1^\pm&=\psi(\pm bx),\\
 z_{\ell+1}^\pm&=\psi(abz_\ell^\pm),
       &&1\leq\ell<L,\\
 q_\ell&=q_{\ell-1}
       +\alpha a\beta^{\ell-1}(z_\ell^++z_\ell^-),
       &&1\leq\ell\leq L.
 \end{aligned}
 \label{eq:notes-square-chains}
\end{equation}
For each input, all neuron outputs in one branch are zero. Induction gives
\(\varphi^{\circ\ell}(x)=a(z_\ell^++z_\ell^-)\) and hence
\(q_\ell=Q_\ell(x)\). Thus \(q_L\) is the final output, while
\(q_1,\ldots,q_{L-1}\) are linear hidden neurons.
The branches and the running sum each pass across one intervening layer
through skip connections.

\begin{theorem}[Double-exponential square approximation]
\label{thm:notes-square-network}
Under \eqref{eq:notes-square-alpha}, the network \(Q_L\) has depth
\(2L-1\), width \(2\), and size at most \(5L-1\). It has no
biases, and all its weights are bounded by
\begin{equation}
 B_{\alpha,\beta}:=\max\{1,b,ab,\alpha a\},
 \label{eq:notes-weight-bound}
\end{equation}
independently of \(L\).  Moreover,
\begin{equation}
 0\leq x^2-Q_L(x)
 =\beta^L\bigl(\varphi^{\circ L}(x)\bigr)^2
 \leq\varepsilon_L,
 \qquad |x|\leq1.
 \label{eq:notes-square-network-error}
\end{equation}
Consequently, uniform error \(\varepsilon\) is attained with depth and
size \(O(\log\log(1/\varepsilon))\) as \(\varepsilon\downarrow0\),
for the fixed activation \(\psi\) and fixed parameters \((\alpha,\beta)\).
\end{theorem}

\begin{proof}
There are \(L\) nonlinear layers of width two and \(L-1\) linear
hidden layers of width one, followed by the scalar output \(q_L\).
Thus the depth is \(2L-1\) and the width is two. The nonlinear neurons
use \(2+2(L-1)=2L\) weights. The accumulators use \(2L\) weights on
the branch outputs and \(L-1\) unit weights on earlier sums, giving
size at most \(5L-1\). The zero initialization \(q_0\) needs no neuron
or connection. All coefficients are among
\(\pm b\), \(ab\), \(1\), and \(\alpha a\beta^{\ell-1}\), so the
weight bound \eqref{eq:notes-weight-bound} holds.

The equality in \eqref{eq:notes-square-network-error} follows from
\eqref{eq:notes-square-fe} and the exact remainder
\eqref{eq:notes-exact-remainder}.  Induction using
\eqref{eq:notes-square-fe} gives
\begin{equation}
 0\leq\varphi^{\circ L}(x)
 \leq\alpha^{-(2^L-1)}|x|^{2^L},\qquad |x|\leq1.
 \label{eq:notes-iterate-decay}
\end{equation}
This proves the upper bound.  For \(0<\varepsilon<1\), it is enough to
take
\[
 L=\max\left\{1,
 \left\lceil\log_2\left(1+\frac{\log(1/\varepsilon)}{2\log\alpha}\right)
 \right\rceil\right\},
\]
because \(\beta^L\leq1\).  The asserted depth and size estimates follow.
\end{proof}

Since \(\varphi\) also preserves \([0,1]\), the same error estimate holds
on the interval used for the piecewise linear comparison in
Section~\ref{sec:GWC}. We use \([-1,1]\) to accommodate the signed
arguments \((u\pm v)/2\) in the multiplication construction of
Section~\ref{sec:Approx}.

The improved rate comes from quadratic convergence to the fixed point
\(0\), not from varying \(\beta\) with depth.  For example, the fixed
choice \((\alpha,\beta)=(2,1/2)\) gives
\(\varepsilon_L=2^{-L-2(2^L-1)}\) and \(B_{\alpha,\beta}=4\).
Since \(\varphi\) is even and increasing on \([0,1]\), the worst-case
error is attained at \(x=\pm1\):
\[
 \|x^2-Q_L\|_{L^\infty([-1,1])}
 =\beta^L[\varphi^{\circ L}(1)]^2.
\]
Table~\ref{tab:square-errors} illustrates the difference between the
geometric bound and the bound obtained by tracking the iterates.
The weight bound depends on the fixed pair \((\alpha,\beta)\);
it is not uniform as \(\beta\downarrow0\).

\begin{table}[htbp]
 \centering
 \caption{Square approximation with fixed \((\alpha,\beta)=(2,1/2)\)
 and weights bounded by four. Values are rounded to five
 significant digits. The last column evaluates the exact remainder
 formula in high precision; it is not the error of a floating-point
 evaluation of the truncated sum.}
 \label{tab:square-errors}
 \begin{tabular}{c@{\qquad}c@{\qquad}c@{\qquad}c}
 \hline
 \(L\) & Geometric bound \(\beta^L\) &
 Bound \(\varepsilon_L\) & Exact norm error\\
 \hline
 1 & \(5.0000\times10^{-1}\) & \(1.2500\times10^{-1}\) & \(1.0102\times10^{-1}\)\\
 2 & \(2.5000\times10^{-1}\) & \(3.9063\times10^{-3}\) & \(2.4300\times10^{-3}\)\\
 3 & \(1.2500\times10^{-1}\) & \(7.6294\times10^{-6}\) & \(2.9453\times10^{-6}\)\\
 4 & \(6.2500\times10^{-2}\) & \(5.8208\times10^{-11}\) & \(8.6750\times10^{-12}\)\\
 5 & \(3.1250\times10^{-2}\) & \(6.7763\times10^{-21}\) & \(1.5051\times10^{-22}\)\\
 6 & \(1.5625\times10^{-2}\) & \(1.8367\times10^{-40}\) & \(9.0616\times10^{-44}\)\\
 \hline
 \end{tabular}
\end{table}

The next section uses the explicit square modules \(Q_n\). The square approximation error satisfies a sharper bound for smaller
inputs, with the dependence on input magnitude becoming stronger
as the number of iterations increases. The associated multiplication modules produce outputs no larger in magnitude than the exact products.
Together, these properties allow depth to be allocated according
to factor size in polynomial and analytic approximation.

%% file: figures/square-approximation-plot.tex
\providecommand{\squarefiguredatadir}{figures}
\begin{tikzpicture}
 \pgfplotsset{square approximation axis/.style={
  scale only axis,
  width=0.365\textwidth,
  height=0.20\textwidth,
  xlabel={\(x\)},
  tick label style={font=\footnotesize},
  label style={font=\footnotesize},
  tick align=outside,
  tick style={black!70},
  legend style={font=\footnotesize,at={(0.5,1.04)},anchor=south,
                legend columns=3,draw=none,fill=none,
                inner xsep=0pt,inner ysep=1pt,
                column sep=1em,row sep=2pt,
                cells={anchor=west}},
  legend image post style={xscale=1.25},
  grid=major,
  grid style={gray!25,line width=0.35pt},
  axis line style={black!70},
  line width=1pt}}
 \begin{axis}[
  square approximation axis,
  at={(0,0)},anchor=south west,
  xmin=-1,xmax=1,
  ymin=0,ymax=1.05,
  ylabel={function value},
  xtick={-1,-0.5,0,0.5,1},
  ytick={0,0.25,0.5,0.75,1},
  yticklabels={\(0\),\(0.25\),\(0.50\),\(0.75\),\(1\)}]
  \addplot[black,line width=1.4pt]
   table[x=x,y=target]{\squarefiguredatadir/square-approximation.dat};
  \addlegendentry{\(x^2\)}
  \addplot[red!75!black,densely dashed]
   table[x=x,y=Q1]{\squarefiguredatadir/square-approximation.dat};
  \addlegendentry{\(L=1\)}
  \addplot[blue!75!black,
           dash pattern=on 4.5pt off 1.8pt on 0.8pt off 1.8pt]
   table[x=x,y=Q2]{\squarefiguredatadir/square-approximation.dat};
  \addlegendentry{\(L=2\)}
 \end{axis}
 \begin{axis}[
  square approximation axis,
  at={(0.50\textwidth,0)},anchor=south west,
  xmin=0.05,xmax=1,
  ymode=log,log basis y=10,
  ymin=1e-27,ymax=1,
  ylabel={absolute error},
  xtick={0.2,0.4,0.6,0.8,1},
  ytick={1,1e-5,1e-10,1e-15,1e-20,1e-25}]
  \addplot[red!75!black,densely dashed]
   table[x=x,y=e1]{\squarefiguredatadir/square-errors.dat};
  \addlegendentry{\(L=1\)}
  \addplot[blue!75!black,
           dash pattern=on 4.5pt off 1.8pt on 0.8pt off 1.8pt]
   table[x=x,y=e2]{\squarefiguredatadir/square-errors.dat};
  \addlegendentry{\(L=2\)}
  \addplot[teal!70!black,densely dotted,line width=1.15pt]
   table[x=x,y=e3]{\squarefiguredatadir/square-errors.dat};
  \addlegendentry{\(L=3\)}
 \end{axis}
\end{tikzpicture}

%% file: sections/approximation.tex
\section{Polynomial and analytic approximation}
\label{sec:Approx}

We now use the local square remainder to guide the construction of
polynomial networks. We first control the error and output magnitude
of a multiplication module, then combine factors of comparable degree
to build monomials. On an interior cube, the smaller inputs to
higher-degree products allow shallower modules. This allocation gives
analytic approximation with error \(O(e^{-cL^{1/d}})\),
depth at most \(L\), and size \(O(L)\).

Fix \(0<\beta<1\) and \(\alpha\) satisfying
\eqref{eq:notes-square-alpha}.  Throughout this section, all networks use
the single activation \(\psi\) from Section~\ref{sec:Overcome}.
We use \(n\) for the number of generator iterations in each square
or multiplication module and reserve \(L\) for the total depth budget,
including linear hidden layers. We use \(Q_n\) for the square network
defined in \eqref{eq:notes-square-network},
with depth \(2n-1\), and we write
\begin{equation}
 \varepsilon_n:=\beta^n\alpha^{-2(2^n-1)}
 \leq\alpha^2\exp\bigl(-2\log\alpha\,2^n\bigr),
 \qquad n\geq1.
 \label{eq:notes-module-epsilon}
\end{equation}
The exact remainder and \eqref{eq:notes-iterate-decay} give the stronger
pointwise estimate
\[
 0\leq e_n(t):=t^2-Q_n(t)=\beta^{n}\left(\varphi^{\circ n}(t)\right)^2
 \leq\varepsilon_n|t|^{2^{n+1}},\qquad |t|\leq1.
\]
Here \(0<\varepsilon_n<1\), and increasing \(n\) also raises the
order of vanishing of the error at the origin. We will use this local
behavior together with control of intermediate magnitudes to assign
depth to each multiplication module.
We use the depth, width, and size conventions of
Definition~\ref{def:notes-network-model}. Width counts all neurons
in each hidden layer, including linear neurons. Later modules may
reuse earlier outputs through skip connections, with every nonzero
connection weight included in the size. Thus width four below does
not bound the number of values retained for reuse or the width needed
to carry them through consecutive layers.

\subsection{Stable multiplication and polynomial networks}

Define
\begin{equation}
 \mathcal M_n(u,v)
 :=Q_n\!\left(\frac{u+v}{2}\right)
   -Q_n\!\left(\frac{u-v}{2}\right).
 \label{eq:notes-Mn}
\end{equation}
Using the constants \(a,b\) from
\eqref{eq:notes-smooth-activation}, every multiplication module has
the same repeated layer pattern:
\begin{equation}
 \begin{gathered}
 C(u,v)=(u+v,-u-v,u-v,-u+v)^{\mathsf T},\\
 \boldsymbol z_1=\psi\!\left(\frac b2 C(u,v)\right),\qquad
 \boldsymbol z_{\ell+1}=\psi(ab\boldsymbol z_\ell),
 \quad 1\leq\ell<n,\\
 m_0=0,\qquad
 m_\ell=m_{\ell-1}+\alpha a\beta^{\ell-1}
 (z_{\ell,1}+z_{\ell,2}-z_{\ell,3}-z_{\ell,4}),
 \quad 1\leq\ell\leq n.
 \end{gathered}
 \label{eq:notes-multiplication-layers}
\end{equation}
Here \(\psi\) acts componentwise and
\(m_n=\mathcal M_n(u,v)\). Four-neuron nonlinear layers
\(\boldsymbol z_\ell\) alternate with scalar linear accumulators
\(m_\ell\). Within a module, each update of
\(\boldsymbol z_\ell\) skips one linear layer, and each update of
\(m_\ell\) skips one nonlinear layer. The two inputs enter only the
first nonlinear layer; the final accumulator supplies the product
to later modules. Thus changing \(n\) only changes the number of
repetitions of this fixed pair of layers. As a standalone network,
the module has \(m_n\) as its affine output. When used inside a
larger network, this output becomes a scalar linear hidden layer.

The one-sided square error controls both the error and the magnitude
of the approximate product.

\begin{lemma}[Stable multiplication]
\label{lem:notes-multiplication}
For \(u,v\in[-1,1]\),
\begin{equation}
 \begin{gathered}
 uv\,\mathcal M_n(u,v)\geq0,
 \qquad |\mathcal M_n(u,v)|\leq|uv|,\\
 |uv-\mathcal M_n(u,v)|
 \leq\varepsilon_n
 \left(\frac{|u|+|v|}{2}\right)^{2^{n+1}}
 \leq\varepsilon_n.
 \end{gathered}
 \label{eq:notes-Mn-error}
\end{equation}
The module has depth \(2n-1\), width at most four,
and size at most \(9n+3\).  All its weights and biases are bounded by
\(B_{\alpha,\beta}\), the constant in
Theorem~\ref{thm:notes-square-network}.
\end{lemma}

\begin{proof}
Put \(s=(u+v)/2\) and \(t=(u-v)/2\). Both arguments belong to
\([-1,1]\), and
\[
 uv-\mathcal M_n(u,v)=e_n(s)-e_n(t).
\]
Both \(Q_n\) and \(e_n=\beta^n(\varphi^{\circ n})^2\) are even and
increasing on \([0,1]\). If \(uv\geq0\), then \(|s|\geq|t|\), so
\(\mathcal M_n(u,v)\geq0\) and
\(uv-\mathcal M_n(u,v)\geq0\). If \(uv\leq0\), both inequalities
reverse. This proves the sign and magnitude assertions.
The pointwise square error gives
\[
 |e_n(s)-e_n(t)|
 \leq\max\{e_n(s),e_n(t)\}
 \leq\varepsilon_n
 \left(\frac{|u|+|v|}{2}\right)^{2^{n+1}}.
\]
The realization \eqref{eq:notes-multiplication-layers} has \(n\)
nonlinear hidden layers and \(n-1\) linear hidden layers, so its
depth is \(2n-1\) and its width is at most four.
The first nonlinear layer uses at most eight weights,
its successors use \(4(n-1)\), the contributions to the accumulators
use \(4n\), and the accumulator connections use \(n-1\). Their sum is
\(9n+3\). For distinct input sources, the weights belong to
\(\{\pm b/2,ab,\pm\alpha a\beta^{\ell-1},1\}\), so their magnitudes
are at most \(B_{\alpha,\beta}\), and all biases are zero.
If the same earlier output is used for both \(u\) and \(v\), the
first-layer inputs reduce to \(bu,-bu,0,0\), which obey the same
bounds.
\end{proof}

Fix an integer \(d\geq1\), and write
\(\mathbb Z_+=\{0,1,2,\ldots\}\). Let \(\boldsymbol e_i\) denote the
\(i\)-th standard basis vector of \(\mathbb R^d\).
For \(\boldsymbol{k}\in\mathbb Z_+^d\), write
\(|\boldsymbol{k}|=k_1+\cdots+k_d\) and
\(\boldsymbol{x}^{\boldsymbol{k}}=x_1^{k_1}\cdots x_d^{k_d}\).
For each \(\boldsymbol k\) of total degree \(j=|\boldsymbol k|\geq2\),
the two input factors are fixed by
\begin{equation}
 k_i^-:=\min\!\left\{k_i,
 \max\!\left(0,\left\lfloor\frac j2\right\rfloor
                  -\sum_{r<i}k_r\right)\right\},
 \quad 1\leq i\leq d,
 \qquad \boldsymbol k^+:=\boldsymbol k-\boldsymbol k^-.
 \label{eq:notes-factor-indices}
\end{equation}
This rule assigns the first \(\lfloor j/2\rfloor\) coordinate factors
of \(\boldsymbol x^{\boldsymbol k}\) to
\(\boldsymbol x^{\boldsymbol k^-}\), in coordinate order. In particular,
\[
 |\boldsymbol k^-|=\lfloor j/2\rfloor,
 \qquad |\boldsymbol k^+|=\lceil j/2\rceil.
\]
Starting with \(X_{\boldsymbol e_i,n}(\boldsymbol{x})=x_i\), define
the remaining monomial approximants in increasing total degree by
\begin{equation}
 X_{\boldsymbol{k},n}(\boldsymbol{x})
 =\mathcal M_n\!\left(
 X_{\boldsymbol k^-,n}(\boldsymbol{x}),
 X_{\boldsymbol k^+,n}(\boldsymbol{x})\right).
 \label{eq:notes-multivariate-monomials}
\end{equation}
Both factors have smaller total degree, so their approximations have
already been computed. A factor may be reused in several products,
including twice in the same product.

For an integer \(p\geq2\), the number of multiplication modules is
\begin{equation}
 M_{d,p}=\sum_{j=2}^{p}\binom{d+j-1}{d-1}
 =\binom{d+p}{d}-d-1.
 \label{eq:notes-number-monomials}
\end{equation}
We place the multiplication modules one after another, in increasing
total degree. Within each degree, we use lexicographic order: at the
first coordinate where two multi-indices differ, the smaller entry
comes first. Denote the resulting order by \(\prec\).

For example, take \(d=2\), \(p=3\), and use \(n\) iterations in every module.
The seven multiplication modules occur in the following order:
\[
 \begin{array}{c|ccc|cccc}
  \text{Module} & 1 & 2 & 3 & 4 & 5 & 6 & 7\\ \hline
  \boldsymbol k & (0,2) & (1,1) & (2,0)
    & (0,3) & (1,2) & (2,1) & (3,0)\\
  \text{Target} & x_2^2 & x_1x_2 & x_1^2
    & x_2^3 & x_1x_2^2 & x_1^2x_2 & x_1^3
 \end{array}
\]
The first three modules approximate the degree-two monomials.
The fourth and fifth modules both reuse the first module's output:
\[
 \begin{aligned}
  X_{(0,3),n}&=\mathcal M_n\bigl(x_2,X_{(0,2),n}\bigr),\\
  X_{(1,2),n}&=\mathcal M_n\bigl(x_1,X_{(0,2),n}\bigr).
 \end{aligned}
\]
The sixth and seventh similarly use the outputs of the second and
third modules, respectively, together with \(x_1\).
If \(n=2\), the successive modules occupy hidden layers
\(1\)--\(4\), \(5\)--\(8\), \(\ldots\), \(25\)--\(28\),
giving depth \(28\) and width at most four.

Let \(\nu_j\geq1\) be the integer iteration count assigned to each
module of degree \(j\). The polynomial construction above uses
\(\nu_j=n\), while the analytic construction below allows this count
to depend on \(j\). Each embedded module consists of \(\nu_j\) pairs
of hidden layers, with four nonlinear neurons followed by one linear
neuron in each pair. The last linear neuron gives the module output.
Placing these modules sequentially keeps the overall width at most four.

The first nonlinear layer of each module receives the two factor
outputs specified by \eqref{eq:notes-factor-indices}. A factor of
degree one is supplied directly by the corresponding input coordinate.
The remaining layers follow the fixed pattern in
\eqref{eq:notes-multiplication-layers}, with internal skip connections
crossing one intervening layer. Once computed, a module output can
be reused by later products and by the final affine output.
Connections between modules may cross several layers, but their
endpoints are determined by the factor rule and the module order.
Thus, for fixed \(d,p\), and \((\nu_j)_{j=2}^p\), the module
connections are fixed independently of the input values and
polynomial coefficients. The latter enter only through the final
affine output.

For completeness, the layer positions can be written explicitly.
The number of hidden layers before module \(\boldsymbol k\) is
\begin{equation}
 s_{\boldsymbol k}:=
 2\!\sum_{\substack{2\leq|\boldsymbol h|\leq p\\
                    \boldsymbol h\prec\boldsymbol k}}
 \nu_{|\boldsymbol h|}.
 \label{eq:notes-module-layer-offset}
\end{equation}
Its \(\ell\)-th nonlinear--linear pair occupies layers
\(s_{\boldsymbol k}+2\ell-1\) and \(s_{\boldsymbol k}+2\ell\),
respectively, for \(1\leq\ell\leq\nu_{|\boldsymbol k|}\).
In particular, its output is the linear neuron in layer
\(s_{\boldsymbol k}+2\nu_{|\boldsymbol k|}\).

\begin{theorem}[Polynomial approximation]
\label{thm:notes-multivariate-polynomial}
Let \(d,n\geq1\) and \(p\geq2\) be integers, and let
\(P(\boldsymbol{x})=\sum_{|\boldsymbol{k}|\leq p}
a_{\boldsymbol{k}}\boldsymbol{x}^{\boldsymbol{k}}\).
The network
\begin{equation}
 \mathcal P_{p,n}(\boldsymbol{x})
 =a_{\boldsymbol0}+\sum_{i=1}^d a_{\boldsymbol e_i}x_i
  +\sum_{2\leq|\boldsymbol{k}|\leq p}
       a_{\boldsymbol{k}}X_{\boldsymbol{k},n}(\boldsymbol{x})
 \label{eq:notes-multivariate-polynomial-network}
\end{equation}
has width at most four, depth at most
\(2nM_{d,p}\), and size at most
\begin{equation}
 (9n+4)M_{d,p}+d+1.
 \label{eq:notes-multivariate-polynomial-size}
\end{equation}
Its weights and biases have magnitude at most
\(\max\{B_{\alpha,\beta},\max_{|\boldsymbol{k}|\leq p}
|a_{\boldsymbol{k}}|\}\), and
\begin{equation}
 \|P-\mathcal P_{p,n}\|_{L^\infty([-1,1]^d)}
 \leq\varepsilon_n
 \sum_{2\leq|\boldsymbol{k}|\leq p}
 (|\boldsymbol{k}|-1)|a_{\boldsymbol{k}}|.
 \label{eq:notes-multivariate-polynomial-error}
\end{equation}

In other words, for fixed \(d,p\) and \(P\), there exist compositional
networks of depth at most \(L\) whose uniform error on \([-1,1]^d\)
is \(O(\exp[-(\log\alpha)\,2^{L/(2M_{d,p})}])\) as
\(L\to\infty\). This is a doubly exponential error bound in the
total depth \(L\).
\end{theorem}

\begin{proof}
Induction in \(|\boldsymbol{k}|\), using
Lemma~\ref{lem:notes-multiplication}, gives
\begin{equation}
 \begin{gathered}
 |X_{\boldsymbol{k},n}(\boldsymbol x)|
 \leq|\boldsymbol x^{\boldsymbol k}|,
 \qquad \boldsymbol x\in[-1,1]^d,\\
 \|X_{\boldsymbol{k},n}-\boldsymbol{x}^{\boldsymbol{k}}\|_{L^\infty([-1,1]^d)}
 \leq(|\boldsymbol{k}|-1)\varepsilon_n.
 \end{gathered}
 \label{eq:notes-multivariate-monomial-error}
\end{equation}
The magnitude bound follows directly from
\(|\mathcal M_n(u,v)|\leq|uv|\). For the error, let
\(E_{\boldsymbol k}\) denote the uniform error for a monomial.
The two factor approximations and the exact monomials have magnitude
at most one, so
\[
 E_{\boldsymbol k}
 \leq E_{\boldsymbol k^-}+E_{\boldsymbol k^+}+\varepsilon_n
 \leq(|\boldsymbol k|-1)\varepsilon_n.
\]
Taking the weighted sum proves
\eqref{eq:notes-multivariate-polynomial-error}.
By Lemma~\ref{lem:notes-multiplication}, each multiplication module
uses at most \(9n+3\) nonzero weights and biases, including the
weights on its two factor inputs. Previously computed factors are
reused through these connections, without copying the modules that
produced them. Hence the \(M_{d,p}\) modules together use at most
\((9n+3)M_{d,p}\) nonzero weights and biases.

Retaining their affine outputs as linear hidden neurons adds no
coefficients. The final affine output contributes at most
\(M_{d,p}+d+1\) additional coefficients, giving
\[
 S\leq(9n+3)M_{d,p}+M_{d,p}+d+1
   =(9n+4)M_{d,p}+d+1.
\]
Each embedded module contributes \(2n\) hidden layers.
Their sequential arrangement therefore gives depth at most
\(2nM_{d,p}\) and width at most four.

For the statement in terms of total depth, take
\(n=\lfloor L/(2M_{d,p})\rfloor\). When \(L\geq2M_{d,p}\),
we have \(2nM_{d,p}\leq L\) and
\(2^n\geq2^{L/(2M_{d,p})-1}\).
Substituting this into \eqref{eq:notes-module-epsilon} and
\eqref{eq:notes-multivariate-polynomial-error} gives the stated rate.
\end{proof}

The polynomial construction above uses the same iteration count \(n\)
in every multiplication module. For analytic approximation, we also
choose the polynomial truncation degree \(p\). The total error then
has two contributions: the truncation error and the error in realizing
the polynomial by a network. We choose \(p\) and the module depths
together so that the network error does not affect the convergence
rate of the polynomial truncation, while keeping the total depth small.
On an interior cube, the magnitude estimate in
\eqref{eq:notes-multivariate-monomial-error} and the local square
remainder allow higher-degree products to use shallower modules.
This leads to the degree-dependent allocation below.

\subsection{Power-series targets and balanced depth allocation}

We consider power series with absolutely summable coefficients on
\([-1,1]^d\) and approximate on \([-\delta,\delta]^d\), where
\(0<\delta<1\). This setting includes functions that extend
holomorphically to a complex polydisk centered at the origin with
all radii greater than one. Truncation at degree \(p\) gives a tail
of order \(\delta^{p+1}\). We choose the module depths so that the
accumulated multiplication error does not affect this geometric rate.

The key is that the local error in \eqref{eq:notes-Mn-error} depends
on the sizes of the two factors. For a product of total degree \(j\),
the balanced factorization uses degrees \(\lfloor j/2\rfloor\) and
\(\lceil j/2\rceil\). The magnitude bound in
Lemma~\ref{lem:notes-multiplication} ensures that both approximate
factors have magnitude at most \(\delta^{\lfloor j/2\rfloor}\).
Thus, writing \(m=\lfloor j/2\rfloor\),
\[
 |uv-\mathcal M_n(u,v)|
 \leq\varepsilon_n\delta^{m2^{n+1}}.
\]
Since \(\varepsilon_n<1\), the condition \(m2^n\geq p\) makes this
local error at most \(\delta^{2p}\). For each integer \(p\geq2\),
we therefore assign
\[
 n_{j,p}:=\left\lceil
 \log_2\frac{p}{\lfloor j/2\rfloor}
 \right\rceil,\qquad 2\leq j\leq p.
\]
These positive iteration counts are nonincreasing in \(j\).
Low-degree products receive more iterations, while products of degree
comparable to \(p\) need only a bounded number. As the proof below
shows, the accumulated error in each monomial of degree at most \(p\)
is bounded by \((p-1)\delta^{2p}=O(\delta^{p+1})\) for fixed \(\delta\).
Thus realizing the truncated polynomial by a network preserves the
overall \(O(\delta^{p+1})\) error bound.

Each module contributes \(2n_{j,p}\) hidden layers.
We use the same factorization rule as in
\eqref{eq:notes-factor-indices} and the same module order,
with \(\nu_j=n_{j,p}\) in
\eqref{eq:notes-module-layer-offset}.

Write \(X_{\boldsymbol k}^{(p)}\) for the resulting monomial approximants.
Figure~\ref{fig:degree-connections} shows the one-dimensional factor
rule and the depth allocation for \(p=8\).
\begin{figure}[!htb]
 \centering
 \includegraphics[width=0.84\linewidth]{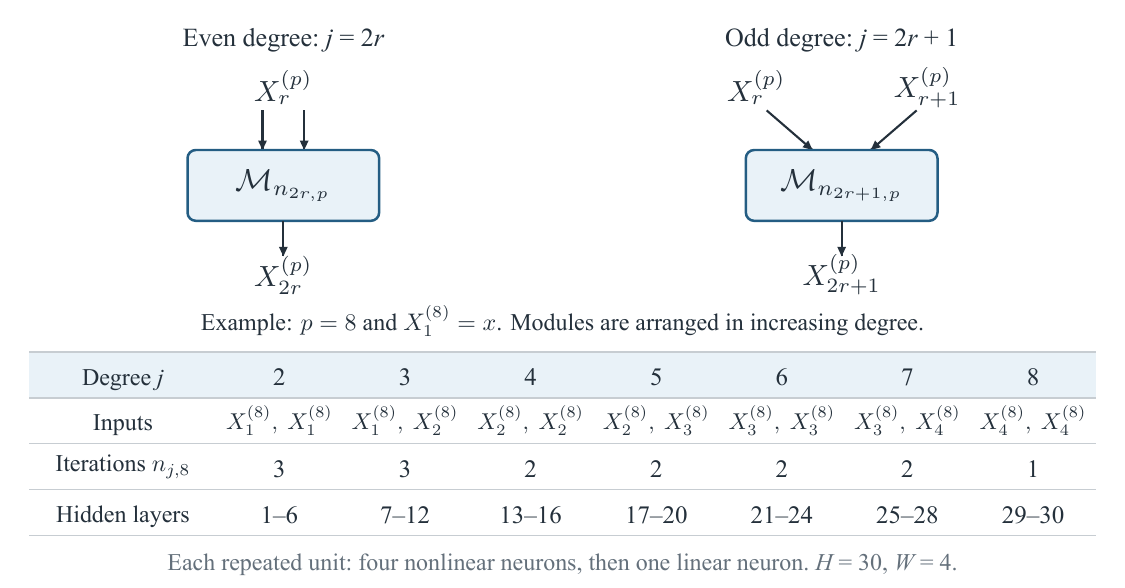}
 \caption{Depth allocation in one dimension. The upper diagrams show
 the general factor rule, and the table gives \(p=8\).
 Here \(X_j^{(p)}\) approximates \(x^j\).
 The iteration counts \(n_{j,8}\) decrease from three to one as the
 degree increases. Each module contributes \(2n_{j,8}\) hidden layers,
 giving total depth \(30\) and width at most four.}
 \label{fig:degree-connections}
\end{figure}

These approximants satisfy
\[
 \begin{aligned}
 X_{\boldsymbol e_i}^{(p)}(\boldsymbol x)&=x_i,\\
 X_{\boldsymbol k}^{(p)}(\boldsymbol x)
 &=\mathcal M_{n_{|\boldsymbol k|,p}}\!\left(
 X_{\boldsymbol k^-}^{(p)}(\boldsymbol x),
 X_{\boldsymbol k^+}^{(p)}(\boldsymbol x)\right),
 \qquad 2\leq|\boldsymbol k|\leq p.
 \end{aligned}
\]
The total depth is at most \(2T_{d,p}\), where
\[
 T_{d,p}:=\sum_{j=2}^{p}
 \binom{d+j-1}{d-1}n_{j,p}.
\]
The theorem below shows that \(T_{d,p}\leq4M_{d,p}=O(p^d)\).
Thus the average number of iterations per module stays bounded,
even though the low-degree modules become deeper as \(p\) increases.

\begin{samepage}
\begin{theorem}[Analytic approximation]
\label{thm:notes-multivariate-analytic}
Suppose
\begin{equation}
 f(\boldsymbol{x})=\sum_{\boldsymbol{k}\in\mathbb Z_+^d}
 a_{\boldsymbol{k}}\boldsymbol{x}^{\boldsymbol{k}},
 \qquad
 A_f:=\sum_{\boldsymbol{k}\in\mathbb Z_+^d}|a_{\boldsymbol{k}}|<\infty
 \label{eq:notes-analytic-assumption}
\end{equation}
on \([-1,1]^d\), and fix \(0<\delta<1\).
For each integer \(p\geq2\), the network
\[
 \mathcal A_p(\boldsymbol x)
 :=a_{\boldsymbol0}
   +\sum_{1\leq|\boldsymbol k|\leq p}
       a_{\boldsymbol k}X_{\boldsymbol k}^{(p)}(\boldsymbol x)
\]
satisfies
\begin{equation}
 \|f-\mathcal A_p\|_{L^\infty([-\delta,\delta]^d)}
 \leq\frac{A_f}{1-\delta}\,\delta^{p+1}.
 \label{eq:notes-multivariate-analytic-basic}
\end{equation}
Its depth \(H\), width \(W\), and size \(S\) obey
\[
 \begin{gathered}
 H\leq2T_{d,p}\leq8M_{d,p},\qquad W\leq4,\\
 S\leq9T_{d,p}+4M_{d,p}+d+1\leq13T_{d,p}+d+1.
 \end{gathered}
\]

Consequently, for every sufficiently large integer \(L\),
there exists a compositional network \(\mathcal N_L\) with depth at most \(L\),
width at most \(4\), size at most \(\frac{13}{2}L+d+1\), and
\begin{equation}
 \begin{gathered}
 \|f-\mathcal N_L\|_{L^\infty([-\delta,\delta]^d)}
 \leq C A_f e^{-cL^{1/d}},\\
 C=\frac{\delta^{-d}}{1-\delta},\quad
 c=\log(1/\delta)\left(\frac{d!}{8}\right)^{1/d}.
 \end{gathered}
 \label{eq:notes-multivariate-rate}
\end{equation}
All these networks have weights and biases bounded by
\(\max\{B_{\alpha,\beta},A_f\}\), independently of \(p\) and \(L\).
\end{theorem}
\end{samepage}

\begin{proof}
\emph{1. Approximation error.}
Induction in total degree, using the magnitude bound in
Lemma~\ref{lem:notes-multiplication}, gives
\[
 |X_{\boldsymbol k}^{(p)}(\boldsymbol x)|
 \leq|\boldsymbol x^{\boldsymbol k}|
 \leq\delta^{|\boldsymbol k|},
 \qquad \boldsymbol x\in[-\delta,\delta]^d.
\]
For a product of degree \(j\), both approximate factors therefore
have magnitude at most \(\delta^m\), where \(m=\lfloor j/2\rfloor\).
Our choice \(m2^{n_{j,p}}\geq p\) and
\eqref{eq:notes-Mn-error} bound the local multiplication error by
\[
 \varepsilon_{n_{j,p}}\,
 \delta^{m2^{n_{j,p}+1}}
 \leq\delta^{2p}.
\]
Let \(E^{(p)}_{\boldsymbol k}\) be the uniform error in
\(X_{\boldsymbol k}^{(p)}\) on \([-\delta,\delta]^d\).
Errors inherited from the two factors are multiplied by quantities
of magnitude at most one. Since the degree-one factors are exact,
induction gives
\[
 E^{(p)}_{\boldsymbol k}
 \leq E^{(p)}_{\boldsymbol k^-}+E^{(p)}_{\boldsymbol k^+}+\delta^{2p}
 \leq(|\boldsymbol k|-1)\delta^{2p}.
\]
The same estimate applies when a factor is used twice. Taking the
weighted sum over monomials and adding the omitted series terms yields
\[
 \|f-\mathcal A_p\|_{L^\infty([-\delta,\delta]^d)}
 \leq\underbrace{A_f(p-1)\delta^{2p}}_{\text{network error}}
   +\underbrace{A_f\delta^{p+1}}_{\text{series tail}}.
\]
Since \((p-1)\delta^{p-1}
 \leq\sum_{k=1}^{p-1}\delta^k
 \leq\delta/(1-\delta)\), this proves
\eqref{eq:notes-multivariate-analytic-basic}.

\emph{2. Network size.}
There are \(m_j=\binom{d+j-1}{d-1}\) modules of degree \(j\).
These counts are nondecreasing in \(j\), while the iteration counts
\(n_{j,p}\) are nonincreasing. The average iteration count over all modules
is therefore no larger than the unweighted average over degrees by Chebyshev's sum inequality:
\[
 \frac{T_{d,p}}{M_{d,p}}
 \leq\frac1{p-1}\sum_{j=2}^{p}n_{j,p}.
\]
Using \(\lfloor j/2\rfloor\geq(j-1)/2\), \(\log((p-1)!)\geq(p-1)\log(p-1)-(p-1)+1\) and
\((p-1)\log(p/(p-1))\leq1\), we obtain
\[
 \begin{aligned}
 \sum_{j=2}^{p}n_{j,p}
 &\le p-1+\sum_{j=2}^{p}\log_2\frac{p}{\lfloor j/2\rfloor}
 \\
 &\le p-1+\log_2p+\sum_{j=3}^{p}\log_2\frac{2p}{j-1} 
 \\
 &\leq2(p-1)+\frac1{\log2}
          \sum_{j=1}^{p-1}\log\frac pj\\
 &\leq\left(2+\frac1{\log2}\right)(p-1)
 <4(p-1).
 \end{aligned}
\]
This proves \(T_{d,p}\leq4M_{d,p}\).

The modules contribute at most \(9T_{d,p}+3M_{d,p}\) nonzero
coefficients, and the final affine output contributes at most
\(M_{d,p}+d+1\). Since every module uses at least one iteration,
\(M_{d,p}\leq T_{d,p}\), giving the stated size bounds.
The width and coefficient bounds follow from the same construction
as in Theorem~\ref{thm:notes-multivariate-polynomial}.

\emph{3. Choosing the degree for a given depth.}
Since the error decays geometrically in \(p\) and the depth is
\(O(p^d)\), we take \(p\) of order \(L^{1/d}\). Specifically, set
\begin{equation}
 R_L=\left(\frac{d!L}{8}\right)^{1/d},
 \qquad p_L=\lfloor R_L\rfloor-d.
 \label{eq:notes-analytic-depth-choice}
\end{equation}
For \(p_L\geq2\), which holds whenever
\(L\geq8(d+2)^d/d!\), the network
\(\mathcal N_L=\mathcal A_{p_L}\) satisfies
\[
 2T_{d,p_L}\leq8M_{d,p_L}
 \leq\frac{8(p_L+d)^d}{d!}\leq L.
\]
Its size is at most
\(13T_{d,p_L}+d+1\leq\frac{13}{2}L+d+1\).
Since \(p_L+1\geq R_L-d\),
\eqref{eq:notes-multivariate-analytic-basic} gives, with
\(r=\log(1/\delta)\),
\begin{equation}
 \|f-\mathcal N_L\|_{L^\infty([-\delta,\delta]^d)}
 \leq\frac{A_f\delta^{-d}}{1-\delta}\,e^{-rR_L},
 \label{eq:notes-multivariate-analytic-optimized}
\end{equation}
which is \eqref{eq:notes-multivariate-rate} with the stated constants.
\end{proof}

For fixed \(d,\alpha,\beta,\delta\) and \(0<\varepsilon<1\),
choose \(p\geq2\) so that
\(\delta^{p+1}\leq(1-\delta)\varepsilon\).
This requires only \(p=O(\log(1/\varepsilon))\), so
\eqref{eq:notes-multivariate-analytic-basic} gives error at most
\(A_f\varepsilon\) with depth and size
\begin{equation}
 O\!\left((\log(1/\varepsilon))^d\right),
 \qquad \varepsilon\downarrow0.
 \label{eq:notes-analytic-epsilon-cost}
\end{equation}
In one dimension, \eqref{eq:notes-multivariate-rate} becomes
\(O(e^{-cL})\).

\begin{remark}[Scope of the analytic result]
The polynomial theorem holds on the full cube, but the analytic rate uses
\(0<\delta<1\) both for the coefficient tail and for the decay of
the monomial factors. At \(\delta=1\), absolute summability alone gives
no specified rate for the coefficient tail. Likewise, real analyticity
near the real cube does not ensure convergence of the Taylor series at
the origin throughout that cube.  A full-cube theorem for a general
holomorphic neighborhood requires additional approximation arguments.
An affine change of variables transfers the present result to another
box whenever the transformed target satisfies
\eqref{eq:notes-analytic-assumption} on a larger normalized cube.
\end{remark}

\begin{remark}[Comparison of the resource bounds]
For fixed \(d\), the bound \(S\leq\frac{13}{2}L+d+1\)
also yields an error bound \(O(\exp(-c_{\mathrm{size}}N^{1/d}))\)
under a size budget \(N\), for some \(c_{\mathrm{size}}>0\).
For absolutely convergent power series, E and Wang~\cite{E} obtained
interior-cube approximation by ReLU networks of width \(d+4\), with
an exponential term \(\exp(-cL^{1/(2d)})\) in their depth parameter.
Their fixed width gives \(O_d(L)\) parameters, so the size exponent is
\(1/(2d)\) as well. Their construction uses input and output skip
connections before conversion to a standard feedforward network; its
reported width counts the full hidden layers, as does ours. Our
network retains skip connections between modules, so the two width
bounds refer to different architectures.

For functions with a holomorphic extension to a polyellipse,
Opschoor, Schwab, and Zech~\cite[Theorem~3.6]{Opschoor2022} obtained
ReLU networks of size at most \(N\), depth
\(O(N^{1/(d+1)}\log N)\), and \(W^{1,\infty}\)-error
\(O(\exp(-bN^{1/(d+1)}))\) on the full cube.
Their RePU result~\cite[Theorem~3.10]{Opschoor2022} gives size
\(O(N)\), depth \(O(\log N)\), and \(W^{k,\infty}\)-error
\(O(\exp(-bN^{1/d}))\) for every fixed integer \(k\geq0\).

The present upper bound has a larger size exponent than these ReLU
bounds and the same exponent as the RePU bound. It uses the fixed
globally Lipschitz activation \(\psi\), weights bounded independently
of accuracy, and the uniform norm on an interior cube under
\eqref{eq:notes-analytic-assumption}. The activation, approximation
setting, and depth bounds differ; no matching lower bound for
Definition~\ref{def:notes-network-model} is established here.

The improvement follows from the local square remainder, the magnitude
control of the multiplication module, and the assignment of depths by
degree. Both linear and nonlinear hidden layers contribute to our
depth and width, and the size includes weights on skip connections.
\end{remark}

%% file: sections/conclusion.tex
\section{Conclusion}
\label{sec:Conclusion}

The results connect restrictions on a shared generator with the behavior
of the exact approximation remainder. When a continuous piecewise linear
generator with finitely many pieces supplies both the forcing term and
the composition map, every \(C^3\) output is at most quadratic.
For non-affine quadratics, \(\beta\geq1/4\), and a fixed point with
nonzero target value keeps the uniform norm of
\(f\circ\varphi^{\circ L}\) bounded below by a positive constant.
The iterated target therefore provides no additional uniform decay,
and each fixed representation has truncation error of order \(\beta^L\).

For the square target, we construct a generator satisfying
\(0\leq\varphi(x)\leq x^2/\alpha\), with \(\alpha>1\). Its iterates
drive the composed target in the remainder to zero, giving doubly
exponential accuracy with fixed geometric weighting. The construction
uses one fixed nonpolynomial, globally \(1\)-Lipschitz \(C^1\)
activation and network coefficients bounded independently of accuracy.

Repeated composition gives a square approximation whose error
is controlled by increasingly high powers of the input magnitude. The associated multiplication module preserves
the small magnitudes of intermediate factors. Composing these modules
gives a doubly exponential error bound in total depth for each fixed
polynomial. Together, these properties
allow balanced monomial factorizations to use less depth for
higher-degree products on interior cubes. The average module depth stays
bounded, giving total depth and size \(O(p^d)\) through degree
\(p\). For power series with absolutely summable coefficients on
\([-1,1]^d\), this allocation transfers the exponential truncation decay
to uniform error \(O(e^{-cL^{1/d}})\) on interior cubes, with
depth at most \(L\) and size \(O(L)\). These bounds quantify how local
errors can guide depth allocation within this compositional construction.

Further work could investigate other polynomial bases, including
Chebyshev polynomials, to treat general holomorphic neighborhoods; see
\cite{MasonHandscomb} and \cite[Chapters~3 and~8]{TrefethenATAP}.
Explicit costs for replacing skip connections by consecutive layers and
error bounds for finite-precision evaluation would also clarify how the
construction translates into numerical implementation.

%% file: ref.bib
@article{Barron,
  author  = {Barron, A. R.},
  title   = {Universal approximation bounds for superpositions of a sigmoidal function},
  journal = {IEEE Trans. Inform. Theory},
  volume  = {39},
  number  = {3},
  pages   = {930--945},
  year    = {1993}
}

@article{Cybenko,
  author  = {Cybenko, G.},
  title   = {Approximation by superpositions of a sigmoidal function},
  journal = {Math. Control Signals Syst.},
  volume  = {2},
  number  = {4},
  pages   = {303--314},
  year    = {1989}
}

@article{Daubechies,
  author  = {Daubechies, I. and DeVore, R. and Foucart, S. and Hanin, B. and Petrova, G.},
  title   = {Nonlinear approximation and (deep) {ReLU} networks},
  journal = {Constr. Approx.},
  volume  = {55},
  number  = {1},
  pages   = {127--172},
  year    = {2022}
}

@misc{DoleanMontanelli2026,
  author       = {Dolean, V. and Montanelli, H.},
  title        = {Machine Learning for Scientific Computing and Numerical Analysis},
  howpublished = {Lecture notes, {\'E}cole Polytechnique},
  year         = {2026},
  note         = {HAL: hal-04976856v2},
  url          = {https://hal.science/hal-04976856v2}
}

@book{Despres,
  author    = {Despr{\'e}s, B.},
  title     = {Neural Networks and Numerical Analysis},
  publisher = {De Gruyter},
  address   = {Berlin/Boston},
  year      = {2022}
}

@article{DespresAncellin2020,
  author  = {Despr{\'e}s, B. and Ancellin, M.},
  title   = {A functional equation with polynomial solutions and application to neural networks},
  journal = {C. R. Math.},
  volume  = {358},
  number  = {9--10},
  pages   = {1059--1072},
  year    = {2020}
}

@article{Devore1,
  author  = {DeVore, R. and Hanin, B. and Petrova, G.},
  title   = {Neural network approximation},
  journal = {Acta Numer.},
  volume  = {30},
  pages   = {327--444},
  year    = {2021}
}

@article{E,
  author  = {E, Weinan and Wang, Qingcan},
  title   = {Exponential convergence of the deep neural network approximation for analytic functions},
  journal = {Sci. China Math.},
  volume  = {61},
  number  = {10},
  pages   = {1733--1740},
  year    = {2018},
  doi     = {10.1007/s11425-018-9387-x},
  url     = {https://link.springer.com/article/10.1007/s11425-018-9387-x}
}

@article{Elbrachter,
  author  = {Elbr{\"a}chter, D. and Perekrestenko, D. and Grohs, P. and B{\"o}lcskei, H.},
  title   = {Deep neural network approximation theory},
  journal = {IEEE Trans. Inform. Theory},
  volume  = {67},
  number  = {5},
  pages   = {2581--2623},
  year    = {2021}
}

@book{Goodfellow,
  author    = {Goodfellow, I. and Bengio, Y. and Courville, A.},
  title     = {Deep Learning},
  publisher = {MIT Press},
  address   = {Cambridge, MA},
  year      = {2016}
}

@article{He,
  author  = {He, J. and Li, L. and Xu, J.},
  title   = {{ReLU} deep neural networks from the hierarchical basis perspective},
  journal = {Comput. Math. Appl.},
  volume  = {120},
  pages   = {105--114},
  year    = {2022}
}

@article{Jiao,
  author  = {Jiao, Y. and Wang, Y. and Yang, Y.},
  title   = {Approximation bounds for norm constrained neural networks with applications to regression and {GANs}},
  journal = {Appl. Comput. Harmon. Anal.},
  volume  = {65},
  pages   = {249--278},
  year    = {2023}
}

@book{KuczmaChoczewskiGer,
  author    = {Kuczma, M. and Choczewski, B. and Ger, R.},
  title     = {Iterative Functional Equations},
  series    = {Encyclopedia of Mathematics and its Applications},
  volume    = {32},
  publisher = {Cambridge University Press},
  address   = {Cambridge},
  year      = {1990}
}

@article{LeCun,
  author  = {LeCun, Y. and Bengio, Y. and Hinton, G.},
  title   = {Deep learning},
  journal = {Nature},
  volume  = {521},
  number  = {7553},
  pages   = {436--444},
  year    = {2015}
}

@inproceedings{Liang,
  author    = {Liang, S. and Srikant, R.},
  title     = {Why deep neural networks for function approximation?},
  booktitle = {5th International Conference on Learning Representations ({ICLR})},
  year      = {2017},
  url       = {https://openreview.net/forum?id=SkpSlKIel}
}

@book{MasonHandscomb,
  author    = {Mason, J. C. and Handscomb, D. C.},
  title     = {Chebyshev Polynomials},
  publisher = {Chapman \& Hall/CRC},
  address   = {Boca Raton, FL},
  year      = {2003}
}

@article{Montanelli1,
  author  = {Montanelli, H. and Du, Q.},
  title   = {New error bounds for deep {ReLU} networks using sparse grids},
  journal = {SIAM J. Math. Data Sci.},
  volume  = {1},
  number  = {1},
  pages   = {78--92},
  year    = {2019}
}

@article{Montanelli2,
  author  = {Montanelli, H. and Yang, H.},
  title   = {Error bounds for deep {ReLU} networks using the {Kolmogorov--Arnold} superposition theorem},
  journal = {Neural Netw.},
  volume  = {129},
  pages   = {1--6},
  year    = {2020}
}

@article{Pinkus,
  author  = {Pinkus, A.},
  title   = {Approximation theory of the {MLP} model in neural networks},
  journal = {Acta Numer.},
  volume  = {8},
  pages   = {143--195},
  year    = {1999}
}

@misc{PetersenZech2024,
  author       = {Petersen, P. and Zech, J.},
  title        = {Mathematical Theory of Deep Learning},
  howpublished = {arXiv preprint arXiv:2407.18384},
  year         = {2024},
  eprint       = {2407.18384},
  archivePrefix = {arXiv},
  primaryClass = {cs.LG},
  note         = {Version 4, revised 15 January 2026},
  url          = {https://arxiv.org/abs/2407.18384v4}
}

@article{Poggio,
  author  = {Poggio, T. and Mhaskar, H. and Rosasco, L. and Miranda, B. and Liao, Q.},
  title   = {Why and when can deep-but not shallow-networks avoid the curse of dimensionality: A review},
  journal = {Internat. J. Automat. Comput.},
  volume  = {14},
  number  = {5},
  pages   = {503--519},
  year    = {2017}
}

@article{Shen1,
  author  = {Shen, Z. and Yang, H. and Zhang, S.},
  title   = {Deep network approximation characterized by number of neurons},
  journal = {Commun. Comput. Phys.},
  volume  = {28},
  number  = {5},
  pages   = {1768--1811},
  year    = {2020}
}

@article{Shen2,
  author  = {Shen, Z. and Yang, H. and Zhang, S.},
  title   = {Deep network with approximation error being reciprocal of width to power of square root of depth},
  journal = {Neural Comput.},
  volume  = {33},
  number  = {4},
  pages   = {1005--1036},
  year    = {2021}
}

@article{Shen3,
  author  = {Shen, Z. and Yang, H. and Zhang, S.},
  title   = {Deep network approximation: Achieving arbitrary accuracy with fixed number of neurons},
  journal = {J. Mach. Learn. Res.},
  volume  = {23},
  number  = {276},
  pages   = {1--60},
  year    = {2022}
}

@article{Siegel2023,
  author  = {Siegel, J. W.},
  title   = {Optimal approximation rates for deep {ReLU} neural networks on {Sobolev} and {Besov} spaces},
  journal = {J. Mach. Learn. Res.},
  volume  = {24},
  number  = {357},
  pages   = {1--52},
  year    = {2023}
}

@book{SinghManhas,
  author    = {Singh, R. K. and Manhas, J. S.},
  title     = {Composition Operators on Function Spaces},
  series    = {North-Holland Mathematics Studies},
  volume    = {179},
  publisher = {North-Holland},
  address   = {Amsterdam},
  year      = {1993}
}

@book{TrefethenATAP,
  author    = {Trefethen, L. N.},
  title     = {Approximation Theory and Approximation Practice, {Extended Edition}},
  publisher = {Society for Industrial and Applied Mathematics},
  address   = {Philadelphia, PA},
  year      = {2019}
}

@article{Opschoor2022,
  author  = {Opschoor, J. A. A. and Schwab, C. and Zech, J.},
  title   = {Exponential {ReLU} {DNN} expression of holomorphic maps in high dimension},
  journal = {Constr. Approx.},
  volume  = {55},
  number  = {1},
  pages   = {537--582},
  year    = {2022},
  doi     = {10.1007/s00365-021-09542-5}
}

@article{XuZhang2022,
  author  = {Xu, Y. and Zhang, H.},
  title   = {Convergence of deep convolutional neural networks},
  journal = {Neural Netw.},
  volume  = {153},
  pages   = {553--563},
  year    = {2022}
}

@article{XuZhang2024,
  author  = {Xu, Y. and Zhang, H.},
  title   = {Uniform convergence of deep neural networks with {Lipschitz} continuous activation functions and variable widths},
  journal = {IEEE Trans. Inform. Theory},
  volume  = {70},
  number  = {10},
  pages   = {7125--7142},
  year    = {2024}
}

@article{Yarotsky,
  author  = {Yarotsky, D.},
  title   = {Error bounds for approximations with deep {ReLU} networks},
  journal = {Neural Netw.},
  volume  = {94},
  pages   = {103--114},
  year    = {2017}
}

@article{Zhou1,
  author  = {Zhou, D.-X.},
  title   = {Universality of deep convolutional neural networks},
  journal = {Appl. Comput. Harmon. Anal.},
  volume  = {48},
  number  = {2},
  pages   = {787--794},
  year    = {2020}
}
